\newif\ifarxiv
\arxivtrue

\newif\ificlr
\iclrfalse

\documentclass{article} 

\ificlr
\usepackage{iclr2027_conference,times}
\else
\ifarxiv
\usepackage[margin=2.5cm]{geometry}
\usepackage{palatino}
\usepackage{natbib}
\usepackage{parskip}
\fi\fi

\usepackage{tikz}
\usetikzlibrary{arrows.meta,positioning,calc,decorations.pathreplacing}
\usepackage{preamble}

\usepackage{math_commands}

\title{
Statistical Benefits of Fine-Tuning from Pretrained Initialization in Diagonal Linear Networks
}

\author{
Alexandre Decl\`eves\\
TML, EPFL, Lausanne, Switzerland\\
\texttt{alexandre.decleves@gmail.com}
\ificlr\And\else\ifarxiv\and\fi\fi
Etienne Boursier \\
INRIA, LMO, Université Paris-Saclay, \\
Orsay, France \\
\texttt{etienne.boursier@inria.fr}
\ificlr\And\else\ifarxiv\and\fi\fi
Nicolas Flammarion \\
TML, EPFL, Lausanne, Switzerland\\
\texttt{nicolas.flammarion@epfl.ch}
}

\ifarxiv\date{}\fi

\begin{document}

\setcounter{tocdepth}{3}

\doparttoc 
\faketableofcontents 

\maketitle

\begin{abstract}
Adapting pretrained models to downstream tasks with limited data has become a central paradigm in modern deep learning. Yet, despite its widespread practical success, how fine-tuning leverages information from pretraining remains poorly understood theoretically.
We study fine-tuning from pretrained weights through the lens of sparse linear regression and two-layer diagonal linear networks. In our setting, pretraining provides information through the support (and signs) of the initialization predictor, which may contain coordinates relevant to the downstream task. We show how pretrained information reshapes the implicit bias and training dynamics, and can thereby reduce the sample complexity of recovering the target parameters and support.
In particular, for a clean initialization with correctly inherited signs, we show that the required sample size is comparable to that of a weighted Lasso estimator that explicitly exploits the pretrained support through a suitably chosen regularizer. 
Our results thus show how information encoded in pretrained weights can be implicitly exploited by gradient-based fine-tuning, reducing the amount of data needed to recover a downstream task.
\end{abstract}

\section{Introduction}

Fine-tuning has become a standard approach for leveraging pretrained models on new tasks, from transferring visual representations to adapting language models to instructions~\citep{kornblith2019better,wei2022finetuned}. Rather than starting from scratch, fine-tuning uses information acquired from previous data to guide learning on the downstream task. In practice, this information is encoded in the pretrained parameters, which serve as the initialization for a gradient-based optimization procedure. 
Yet how, and under what conditions, fine-tuning can exploit this initialization to reduce the data required for the downstream task remains poorly understood theoretically. In particular, two fundamental questions arise: \emph{What information encoded in a pretrained model can reduce the amount of data required to learn a downstream task? And when can fine-tuning exploit this information without simultaneously learning spurious features?}

A large body of work has studied the implicit bias of gradient descent: how the model parameterization and optimization dynamics favor particular solutions even in the absence of explicit regularization. 
This perspective has revealed how gradient-based training can favor simple predictors and how initialization shapes the selected solution~\citep{chizat2019lazy,woodworth2020kernel,boursier2022gradient}.
Much of this theory, however, considers small and uninformative initializations, where training begins without knowledge inherited from a previous task. 
Fine-tuning instead starts from a structured predictor that may already identify features relevant to the downstream task. Understanding its statistical benefits therefore requires characterizing how the resulting implicit bias allows useful inherited information to be retained while the remaining signal is learned.

We make this theoretical study of fine-tuning concrete in the setting of sparse linear regression with two-layer diagonal linear networks. Here, pretraining provides a nonzero predictor whose support may already identify relevant coordinates, although their downstream coefficients need not be accurately estimated. Without support information, estimating an $s$-sparse target in $d$ dimensions from $n$ noisy observations has minimax squared error of order $\sigma^2s\log(d/s)/n$ under suitable design conditions~\citep{wainwright2019high}. Suppose instead that pretraining correctly identifies $s-m$ of the relevant coordinates, leaving only $m$ to be discovered. This suggests a statistical gain: the ambient-dimensional logarithmic cost should only be incurred for coordinates whose membership in the target support is unknown, while the cost of estimating the coefficients on all $s$ coordinates remains. We first formalize this benefit through a weighted Lasso benchmark, which explicitly favors the pretrained support and achieves a corresponding support-dependent guarantee.

The weighted Lasso benchmark, however, does not answer the central question of whether gradient-based fine-tuning can extract the same statistical benefit from initialization alone. We show that pretraining indeed modifies the implicit bias through the inherited signed support. In the limit of vanishing layer imbalance, with fixed initial predictor, we characterize the resulting implicit regularizer in terms of the inherited signed support. The penalty favors reusing inherited coordinates with their pretrained signs rather than preserving their coefficient magnitudes: coordinates absent from the pretrained support incur the usual $\ell_1$ cost, whereas inherited coordinates incur no cost when their signs are preserved.  This asymmetry explains why correctly inherited signs are favored, and why incorrect signs and inherited false positives require separate treatment. In the noiseless setting, it yields an exact-recovery guarantee: for a clean, correctly signed initialization, a sample size of order $s+m\log(d)$ suffices, so only the $m$ missing coordinates incur the ambient-dimensional log cost.

The main difficulty arises with noisy observations. Characterizing the interpolating solution selected at convergence is no longer enough: it does not establish whether the trajectory recovers the missing signal before fitting the noise and activating irrelevant coordinates.  We therefore study the fine-tuning path through its limiting saddle-to-saddle dynamics, building on the construction of \citet{pesme2023saddle} while retaining the nonzero pretrained predictor.
Under a Gaussian design and suitable sample-size and signal-strength conditions, we show that the trajectory preserves correctly inherited true coordinates, corrects initially wrong signs, and recovers the remaining signal before introducing any new false positive. We also derive a gradient-based stopping rule, calibrated using an upper bound on the number of missing or incorrectly signed true coordinates and inherited false positives. The stopped trajectory achieves exact support recovery for a clean initialization; with an imperfect initialization, its support contains the true support and can differ from it only through inherited false positives. 
We complement our theoretical results with synthetic experiments illustrating the role of early stopping and the dependence of support recovery on initialization quality (Appendix~\ref{app:experiments}).

\subsection{Related work}

\paragraph{Theoretical analyses of fine-tuning.}
\citet{shachaf2021theoretical} relate the sample complexity and inductive bias of fine-tuning to source--target similarity in linear-teacher models.
\citet{wu2022power} establish excess risk bounds for pretraining and fine-tuning with SGD in linear regression under covariate shift.
\citet{jones-mccormick2025provable} prove sample-complexity improvements from unsupervised pretraining and transfer learning in single-index models.
Other analyses study the distortion of pretrained features under distribution shift \citep{kumar2022finetuning} and characterize language-model fine-tuning through neural tangent kernels \citep{malladi2023kernel,tomihari2024understanding}.
Beyond the fixed-kernel regime, \citet{lauditi2026transfer} analyze transfer learning in infinite-width networks with feature learning during both pretraining and adaptation.

\paragraph{Sparse recovery with prior support information.}
Prior support information can improve sparse recovery through weighted $\ell_1$ minimization, which penalizes likely support coordinates less heavily~\citep{zou2006adaptive,von2007compressed,khajehnejad2009weighted,oymak2012recovery}. \citet{vaswani2010modified,jacques2010short} analyzed a compressed sensing algorithm, assigning zero weights to an estimated support. \citet{friedlander2011recovering} allowed nonzero weights and established recovery conditions depending on the size and accuracy of the estimate.
%
\citet{mansour2017recovery} established weighted null-space conditions and uniform Gaussian recovery guarantees for sufficiently accurate support estimates. \citet{rauhut2016interpolation} developed a weighted-sparsity framework for function interpolation, while \citet{bah2016sample} derived nonuniform Gaussian sample-complexity bounds reflecting the alignment between the weights and the true support. \citet{flinth2016optimal,lian2018weighted} further study the choice of
weights from prior support information, including optimal weighting and
statistical prior-support models. 
Our weighted-Lasso benchmark specializes this literature to the pretrained-support setting and provides a statistical reference for the fine-tuning guarantees.

\paragraph{Implicit bias in diagonal linear networks.}

Diagonal linear networks (DLN) provide a tractable setting for studying implicit bias.
For least-squares regression, \citet{woodworth2020kernel} show that the implicit bias interpolates between $\ell_2$- and $\ell_1$-norm minimization as the initialization scale decreases.
The connection between multiplicative parameterizations and mirror descent provides a useful framework for analyzing this bias and the resulting sparsity-inducing optimization dynamics \citep{ghai2020exponentiated,vaskevicius2020statistical,azulay2021initialization}.
%
%
Beyond characterizing the terminal solution, \citet{vaskevicius2019implicit,zhao2022high} establish
sparse-estimation guarantees for early-stopped gradient descent
from small initialization under restricted isometry assumptions.
In a complementary direction,
\citet{berthier2023incremental,pesme2023saddle,berthier2026incremental} describe the limiting gradient-flow trajectory as a sequence of saddles
in the vanishing-initialization regime. 
In the fine-tuning setting, \citet{lippl2024inductive,anguita2026a}
characterize how pretrained weights shape implicit bias and feature
reuse when the end-to-end predictor is reset to zero before fine-tuning. 
We instead retain a nonzero pretrained predictor and quantify how its support affects the implicit bias and fine-tuning trajectory.


\section{A minimal model of fine-tuning}
\label{sec:setup}

\subsection{Sparse regression and pretrained model}
\label{subsec:regression}

We study fine-tuning in sparse linear regression, where pretraining can provide partial information about the relevant coordinates. 
We observe a design matrix $\mathbf X\in \mathbb R^{n\times d}$ and a response vector $\mathbf y \in \mathbb R^n$ generated as
\[\textstyle
    \mathbf y
    =
    \mathbf X\beta^\star
    +
    \frac{\sigma}{\sqrt n}\varepsilon,
    \qquad
    \mathbf X_j\overset{\mathrm{i.i.d.}}{\sim}
    \mathcal N \big(0,\frac1n I_n\big),  \qquad
    \varepsilon \sim \mathcal N(0,I_n),
\]
where $\mathbf X_j$ denotes the $j$-th column of $\mathbf X$, the noise $\varepsilon$ is independent of $\mathbf X$, $\sigma \geq 0$ and the target $\beta^\star\in\mathbb R^d$ is sparse. We write
$S^\star=\operatorname{supp}(\beta^\star)$ and $s=|S^\star|$. 

The downstream quadratic loss is
\begin{equation}\textstyle
    L(\beta)
    \coloneqq
    \frac12\|\mathbf y-\mathbf X\beta\|_2^2.
\end{equation}
We model pretraining through a predictor $\beta^0 \in\mathbb R^d$, treated as fixed independently of the design matrix and noise, and write $S_{\rm init}=\operatorname{supp}(\beta^0)$. Our analysis takes this predictor as given rather than modeling the pretraining phase itself. 
Intuitively, $\beta^0$ is inherited from a previous pretraining phase on a different but related data distribution, so that its support provides partial information about $S^\star$.

The overlap $S^\star\cap S_{\rm init}$ represents relevant coordinates already identified by pretraining. Their downstream coefficients are not
assumed to be known or accurately estimated, while the  coordinates in $S^\star\setminus S_{\rm init}$ still need to be identified.
In the clean case, $S_{\rm init} \subseteq S^\star$.
More generally, the pretrained support may contain \emph{null} coordinates, i.e., coordinates in $(S^\star)^c$, which are irrelevant to the downstream task.

\subsection{Weighted Lasso}
To quantify the statistical advantage of the pretrained support, we consider the estimation rate of an explicit regularization benchmark, the weighted Lasso~\citep{zou2006adaptive}, defined for $\lambda\geq 0$ and $\alpha\geq1$ by
\[
    \widehat\beta^{\rm WL}
    \in
    {\textstyle\argmin_{\beta\in\mathbb R^d}}
    \big\{ L(\beta) 
        + \lambda \big( \| \beta_{S^c_{\rm init}} \|_1 + \frac{1}{\alpha} \| \beta_{S_{\rm init}} \|_1 \big) 
    \big\}.
\]
The parameter $\alpha$ controls how strongly the pretrained support is favored. 
\begin{proposition}[Weighted-Lasso benchmark]
\label{prop:wl-benchmark}
Let $\delta\in(0,1)$ and suppose $1\leq |S_{\rm init}|\leq d/2$. Choose
$
\alpha_\star^2
=
\frac{\log(4|S_{\rm init}^c|/\delta)}
     {\log(4|S_{\rm init}|/\delta)}$ and $
\lambda
=
C\sigma
\sqrt{
\frac{\log(4|S_{\rm init}^c|/\delta)}{n}
}$\footnote{For $\sigma=0$, $\widehat\beta^{\mathrm{WL}}$
minimizes the same weighted $\ell_1$ penalty subject to $X\beta=y$. On the stated high-probability event, this minimizer is unique
and equals the $\lambda\to0$ limit of the penalized
solutions. 
}, 
for a sufficiently large universal constant $C$. Then there exists a universal constant $C_0$ such that 
if
\[
n
\geq C_0 \left(1+\frac{\sigma^2}{\min_{i\in S^*}(\beta_i^\star)^2}\right)\cdot\left(
|S^\star\setminus S_{\rm init}|
\log\frac{4|S_{\rm init}^c|}{\delta}
+
|S^\star\cap S_{\rm init}|
\log\frac{4|S_{\rm init}|}{\delta}\right),
\]
then with probability at least $1-\delta$, both 
\begin{gather*}
\|\widehat\beta^{\rm WL}-\beta^\star\|_2
\leq C_0
\frac{\sigma}{\sqrt n}
\sqrt{
|S^\star\setminus S_{\rm init}|
\log\frac{4|S_{\rm init}^c|}{\delta}
+
|S^\star\cap S_{\rm init}|
\log\frac{4|S_{\rm init}|}{\delta}
}\\\text{and}\qquad
\operatorname{supp}(\widehat\beta^{\rm WL})=S^\star.
\end{gather*}
\end{proposition}
The bound separates the statistical costs associated with the two regions defined by the pretrained support. 
The coordinates in $S^\star\setminus S_{\mathrm{init}}$ must be identified among the $|S_{\mathrm{init}}^c|$ coordinates not selected by pretraining, and therefore incur a logarithmic factor in $|S_{\mathrm{init}}^c|$. By contrast, coordinates already contained in $S_{\mathrm{init}}$ only incur a logarithmic factor in the size of this smaller candidate set. This rate improves over the usual Lasso rate in which all $s$ active coordinates pay the ambient-dimensional logarithmic cost.
Thus, although the estimator is a single weighted Lasso, its rate reflects two distinct support-identification problems inside and outside the pretrained support. In sparse regimes, this agrees, up to logarithmic refinements, with the natural statistical complexity of sparse estimation with such two-block support information. 
The signal-to-noise prefactor in the sample-size condition accounts for detecting the smallest nonzero coefficient in exact support recovery; the estimation bound alone holds without this prefactor (Appendix~\ref{app:wl-proofs}).

An oracle choice of $\alpha$ depending on $S^\star$ yields a sharper bound (see \Cref{eq:weightlassoideal} in Appendix~\ref{app:wl-proofs}), requiring only $$n\gtrsim |S^\star\setminus S_{\rm init}|\log_+\frac{
4|S_{\rm init}^c\cap(S^\star)^c|
}{\delta}
+
|S^\star\cap S_{\rm init}|\log_+\frac{
4|S_{\rm init}\setminus S^\star|
}{\delta}$$ samples for recovery, where $\log_+(t) = \max(1, \log(t))$. For a clean initialization ($S_{\rm init}\subseteq S^\star$), this oracle choice (given by $\alpha=\infty$) removes the logarithmic cost on inherited true coordinates, leaving only the $|S^\star\setminus S_{\rm init}|$ missing coordinates to incur the ambient-dimensional logarithmic cost.

Proposition~\ref{prop:wl-benchmark} specializes weighted-$\ell_1$ recovery with prior support information to our pretrained-support model. Its proof combines the Gaussian weighted-cone bounds of \citet{bah2016sample} with standard weighted-Lasso estimation
and primal-dual support-recovery arguments. We include a self-contained derivation in Appendix~~\ref{app:wl-proofs} to make the dependence on the pretrained support explicit.

Weighted Lasso exploits the pretrained support through explicit regularization and provides good guarantees for support recovery. In practical fine-tuning of deep learning models, however, one typically relies on gradient-based optimization initialized at the pretrained weights. We therefore use weighted Lasso as a statistical benchmark against which we compare the solution obtained through such gradient-based fine-tuning.

\section{Fine-tuning diagonal linear networks}
\label{sec:imp_bias}
Motivated by practical fine-tuning, we now study gradient-based optimization initialized at pretrained weights. More precisely, we consider diagonal linear networks (DLNs), a simple neural network architecture that nevertheless exhibits nonconvex optimization dynamics. 
\subsection{Parameterization and gradient flow}
\label{subsec:diag}
We parameterize the predictor as a two-layer DLN  $\beta_w=u\odot v$, where $w = (u,v) \in \mathbb R^{2d}$ and $\odot$ denotes the (Hadamard) componentwise multiplication. The fine-tuning training objective is 
\begin{equation}\textstyle
   F(w)\coloneqq L(u\odot v). 
\end{equation}
While the loss $L$ is convex in the end-to-end predictor $\beta_w=u\odot v$,  $F$ is
non-convex in the network parameters~$w$. This simple reparameterization already produces a rich, non-trivial training trajectory. 
We model fine-tuning by training both layers from an initialization representing the pretrained predictor $\beta^0$. As the limiting dynamics of the (stochastic) gradient descent with infinitesimal
step-sizes, we study gradient flow
\[\textstyle
\dot w^\mu_t=-\nabla F(w^\mu_t).
\]
For $\mu>0$, we initialize the weights so that 
\[\textstyle
u^{\mu}(0)\odot v^{\mu}(0)=\beta^0, \qquad  u_i^{\mu}(0)^2-v_i^{\mu}(0)^2=2\mu, \qquad i\in [d]. 
\]
Equivalently, $
u_i^{\mu}(0)^2=\sqrt{(\beta_i^0)^2+\mu^2}+\mu$, and $v_i^{\mu}(0)^2=\sqrt{(\beta_i^0)^2+\mu^2}-\mu$, for $i\in [d]$, with signs chosen so that $u_i^{\mu}(0)v_i^{\mu}(0)=\beta_i^0$.
The first condition ensures that fine-tuning starts from the pretrained predictor, retaining the information acquired before the downstream task. The second controls the imbalance between the two layers. 

This imbalance $u_i^{\mu}(t)^2-v_i^{\mu}(t)^2$ is preserved along the flow and plays a key role in the implicit bias of the dynamics. We study the regime $\mu \to 0$ with $\beta^0$ fixed: the layer imbalance vanishes, not the pretrained predictor. The next section characterizes the limiting dynamics obtained in that regime. 

\subsection{Mirror flow and the limiting fine-tuning path}
\label{sec:mirror-s2s} 

To study how pretraining affects support recovery during fine-tuning, we first describe the mirror-flow dynamics and their limiting saddle-to-saddle trajectory.

\paragraph{Mirror flow and implicit bias.}
Although the network parameter $w^\mu$ follows a nonconvex gradient flow, the end-to-end predictor $\beta^\mu= u^\mu \odot v^\mu$ evolves according to a mirror flow for the convex loss~$L$~\citep{azulay2021initialization}. In our parametrization, the conserved layer imbalance gives
\begin{equation}\label{eq:mirror}
\frac{\df}{\df t}\nabla\phi_\mu(\beta^\mu(t))=-\nabla L(\beta^\mu(t)), \quad \text{where} \quad
\phi_\mu(\beta)\coloneqq
\frac12\sum_{i=1}^d
\Big[
\beta_i\operatorname{arcsinh}\!\Big(\frac{\beta_i}{\mu}\Big)
-\sqrt{\beta_i^2+\mu^2}
\Big]
\end{equation}
is the hyperbolic entropy~\citep{ghai2020exponentiated}. The mirror flow structure provides a direct characterization of the implicit bias of gradient flow and  makes the role of initialization explicit. If the flow converges to an interpolator $\beta_\infty^\mu$, the mirror identity implies
\begin{equation}\label{eq:implicitbias}
        \beta_\infty^\mu = \argmin_{\beta\in\mathbb R^d: \mathbf X \beta = \mathbf y} D_{\phi_\mu}(\beta,\beta^0),
\end{equation}
where $D_{\phi_\mu}(\beta,\beta^0)\coloneqq \phi_\mu(\beta)-\phi_\mu(\beta^0) -\langle\nabla\phi_\mu(\beta^0),\beta-\beta^0\rangle$ is the associated Bregman divergence.
Thus initialization affects the selected interpolator through the reference point of the divergence. The mirror-descent representation also allows us to characterize the limiting fine-tuning trajectory as $\mu\to0$, while keeping the pretrained predictor $\beta^0$ fixed.

\paragraph{Limiting saddle-to-saddle dynamics.}
To obtain a nondegenerate limit of the dynamics as $\mu\to0$, we rescale time. Indeed, writing $\lambda_\mu:=\frac12\log(1/\mu)$, the mirror potential satisfies $\phi_\mu/\lambda_\mu\to\|\cdot\|_1$. We therefore consider the accelerated predictor
$\widetilde\beta^\mu(\tau):=\beta^\mu(\lambda_\mu\tau)$. Integrating \Cref{eq:mirror} gives
\begin{equation}
    \frac{\nabla\phi_\mu(\widetilde\beta^\mu(\tau))}{\lambda_\mu}
    =
    \frac{\nabla\phi_\mu(\beta^0)}{\lambda_\mu}
    -
    \int_0^\tau \nabla L(\widetilde\beta^\mu(s))\,\df s.
\end{equation}
Following the construction of \citet{pesme2023saddle}, we use the formal limit of this identity to describe a piecewise-constant trajectory $\beta^\circ$ satisfying
\begin{equation}\label{eq:limiting-inclusion}
    q(\tau)
    \coloneqq
    q^0-\int_0^\tau\nabla L(\beta^\circ(s))\,\df s
    \in \partial\|\beta^\circ(\tau)\|_1,
    \qquad
    q_i^0=
    \begin{cases}
        \operatorname{sign}(\beta_i^0), & i\in S_{\rm init},\\
        0, & i\notin S_{\rm init}.
    \end{cases}
\end{equation}
The difference from zero initialization is the nonzero initial dual state $q^0$: inherited coordinates start at the boundary of $[-1,1]$ with their pretrained signs, whereas the other coordinates start at its center.

\Cref{eq:limiting-inclusion} determines the constraints on the predictor: if $|q_i|<1$, then $\beta_i=0$; if $q_i=\pm1$ then $\beta_i$ is either zero or has the corresponding sign. The saddle-to-saddle algorithm alternates between least-squares fits subject to these constraints and linear evolution of $q$. Specifically, define
\[
    \mathcal F(q)
    :=\{\beta\in\mathbb R^d:q\in\partial\|\beta\|_1\},
    \qquad
    \beta(q)\in\argmin_{\beta\in\mathcal F(q)}L(\beta).
\]
Starting from $q^0$, the algorithm first computes $\beta(q^0)$: a least-squares refit on the pretrained support, constrained to preserve its signs, but allowed to set coordinates to zero. The predictor then remains constant while the dual variable  evolves according to
$\dot q=-\nabla L(\beta)$. When a moving coordinate of $q$ reaches either boundary, the predictor is refitted on the resulting signed face, and the procedure repeats. These refits may activate or deactivate coordinates. Algorithm~\ref{alg:saddle-to-saddle} gives the complete recursion defining the piecewise constant limit path $\beta^\circ$.  A precise statement of convergence from the accelerated gradient
flow to this path, together with the required assumptions, is given in Appendix~\ref{app:dynamics}. 
\begin{algorithm}[t]
\caption{Saddle-to-saddle dynamics with pretrained initialization}
\label{alg:saddle-to-saddle}
\begin{algorithmic}[1]
\State Initialize $q\gets q^0$ as in \Cref{eq:limiting-inclusion}, and $\tau\gets0$
\State $\beta\gets\argmin_{\beta\in \mathcal F(q)}L(\beta)$
\While{$\nabla L(\beta)\neq0$ and the stopping criterion is not satisfied}
    \State $\displaystyle
    \Delta\gets
    \inf\left\{
        \rho>0:
        \exists i,\; [\nabla L(\beta)]_i \neq 0 \text{ and }
        q_i-\rho[\nabla L(\beta)]_i\in\{-1,+1\}
    \right\}$
    \State $(\tau,q)\gets(\tau+\Delta,\;q-\Delta\cdot\nabla L(\beta))$
    \State $\displaystyle
    \beta\gets\argmin_{\beta\in \mathcal F(q)}L(\beta)$
\EndWhile
\State \Return successive values of $(\tau,\beta,q)$
\end{algorithmic}
\end{algorithm}

\section{Statistical guarantees for fine-tuning}
\label{sec:early-stopped}

Having characterized the fine-tuning dynamics, we now study the solutions they select. We first identify how pretraining modifies the terminal implicit bias, before establishing recovery guarantees under early stopping. 
Let $F_0\coloneqq S_{\rm init}\setminus S^\star$ be the inherited false positive coordinates, and define
\[
R_0\coloneqq
(S^\star\setminus S_{\rm init})
\cup
\left\{
i\in S^\star\cap S_{\rm init}:
\operatorname{sign}(\beta_i^0)\neq\operatorname{sign}(\beta_i^\star)
\right\}.
\]
The set $R_0$ contains the true coordinates that must be learned or relearned due to an initial wrong sign. We write $m\coloneqq|R_0|$ and $f_0\coloneqq|F_0|$. 
\subsection{Recovery in the noiseless setting}
We first consider the noiseless setting $ \mathbf y=\mathbf X\beta^\star$. We can characterize the  bias in the small-$\mu$ regime.
\begin{proposition}[Leading implicit bias]
\label{prop:implicit-bias}
For fixed $\beta,\beta^0\in\R^d$, let $\lambda_\mu\coloneqq\frac12\log(1/\mu)$. As $\mu\to 0$ 
\[\textstyle
\frac{D_{\phi_\mu}(\beta,\beta^0)}{\lambda_\mu} \
{\longrightarrow} \ R_{\beta^0}(\beta)\coloneqq 
\sum_{j\notin S_{\rm init}}|\beta_j|
+
\sum_{i\in S_{\rm init}}
\bigl(|\beta_i|-\operatorname{sign}(\beta_i^0)\beta_i\bigr).
\]
\end{proposition}
Proposition~\ref{prop:implicit-bias} shows that the leading implicit bias induced by pretraining depends on the inherited signed support, rather than on the pretrained coefficient magnitudes. Pretrained magnitudes only appear in the second-order correction given in Appendix~\ref{app:dynamics}. 
Coordinates outside $S_{\rm init}$ incur the usual $\ell_1$ cost. In contrast, on an inherited coordinate, any coefficient preserving the pretrained sign has zero leading-order cost, while reversing the sign incurs a penalty $2|\beta_i|$. The leading penalty therefore favors the inherited signs without requiring inherited coordinates to remain active. 
This sign asymmetry contrasts with the diagonal-network setting of \citet{lippl2024inductive}, where fine-tuning starts from a zero predictor and induces a penalty invariant to coordinate sign changes.
Importantly, this bias does not distinguish between inherited true coordinates and inherited false positives: both receive the same treatment at leading order. In particular, the leading penalty alone does not favor removing false positives inherited from initialization. 

The following result gives conditions under which $\beta^\star$ is its unique minimizer, quantifying how the pretrained initialization affects exact recovery. 
We study the associated interpolation problem~$
    \min_{\beta:\,\mathbf X\beta=\mathbf y} R_{\beta^0}(\beta)$. 
In this case, the leading implicit bias is sufficient for exact recovery.

\begin{theorem}[Noiseless recovery]
\label{thm:noiseless-recovery}
Let $\delta\in(0,1)$ and $\sigma=0$. 
There exists a universal constant $C$ such that if 
\[\textstyle
    n \geq C\big( s + f_0
    + m\log\!\left(d\right) + \ln(1/\delta)\big),
\]
then with probability at least $1-\delta$:
\begin{enumerate}[itemsep=0pt,topsep=0pt]
    \item $\beta^\star$ is the unique minimizer of
$\min_{\beta:\,\mathbf X\beta=\mathbf y} R_{\beta^0}(\beta)$;
\item $\lim_{\mu\to 0} \beta_\infty^\mu = \beta^\star$.
\end{enumerate}
\end{theorem}
Theorem~\ref{thm:noiseless-recovery} establishes exact recovery in the noiseless setting. The true predictor uniquely minimizes the leading term of the implicit-bias objective, and the gradient-flow endpoints converge to it as $\mu\to0$. 
The sample-size requirement reveals that the benefit of pretraining depends on the accuracy of the inherited signed support. 
Indeed, $m=|R_0|$ counts both true coordinates missing from the pretrained support and true coordinates present with an incorrect sign. 
Both contribute to the ambient-dimensional logarithmic term $m\log(d)$, whereas correctly signed inherited true coordinates avoid this cost. This differs from the weighted-Lasso bound, whose ambient-dimensional logarithmic term depends only on $|S^\star\setminus S_{\rm init}|$: its weights exploit support information without using the pretrained signs. 
The remaining term $s+f_0=|S^\star\cup S_{\rm init}|$ represents the dimension cost associated with fitting coefficients on the union of the true and pretrained supports, analogous to least squares once this set is known. Its dependence on $f_0$ also highlights a limitation of the inherited bias. On $S_{\rm init}$, coefficients preserving the pretrained signs incur no penalty, so the regularizer does not encourage sparsity within this candidate set. Accordingly, the bound depends on its full size, including inherited false positives. Weighted Lasso, by retaining a nonzero sparsity penalty within $S_{\rm init}$, can instead exploit sparsity among the inherited coordinates.

With noisy observations, the interpolating solution reached at convergence overfits the training data and therefore generalizes poorly to unseen data. Early stopping is thus necessary: ideally, at some point along the fine-tuning trajectory, the model has recovered the remaining signal without yet activating spurious coordinates. The implicit-bias characterization in \Cref{eq:implicitbias}, however, does not determine whether such an intermediate iterate is reached, as it only characterizes the terminal point of the trajectory.

\subsection{Recovery in the noisy setting: ideal early stopping} \label{subsec:ideal-stopping}

We first consider an ideal stopping rule which stops the first time, if finite, $S^\star$ is included in the estimated support; and returns the corresponding saddle. This rule is not available in practice, but isolates the main statistical question: does the trajectory recover the remaining signal before proposing a new null coordinate?

We state the main result in a balanced regime, for the sake of presentation. The more general conditions and their proofs are deferred to Appendix~\ref{app:s2s-recovery}.
\begin{assumption}[Balanced low-noise regime]
\label{ass:balanced-low-noise}
There exist $a>0$ and $L\ge1$ such that
\[\textstyle
a\le|\beta_i^\star|\le La,\qquad i\in S^\star,
\qquad
\sigma\le La\sqrt m.
\]
\end{assumption}
\ificlr\vspace{-0.5em}\fi
The lower bound on the non-zero coefficients rules out arbitrarily weak signals for which exact support recovery is statistically ill-posed. The bounded range of coordinates and noise level place us in a regime where the remaining support is detectable.
\begin{theorem}[Noisy recovery under ideal early stopping]
\label{thm:oracle-recovery}
Consider \Cref{ass:balanced-low-noise}. There exists $C_L>0$, depending only on $L$, such that, if
\[\textstyle
    n
    \ge
    C_L\Big(
        s+f_0+m^2+mf_0
        +m\log\frac{d}{\delta}
    \Big),
\]
then, with probability at least $1-\delta$, there exists an ideal stopping time $\tau_{\rm oracle}\in\R_+$ such that, for some universal constant $C$,
\ificlr\vspace{-0.5em}\fi
\[
    S^\star
    \subseteq
    \operatorname{supp}(\beta^\circ(\tau_{\rm oracle}))
    \subseteq
    S^\star\cup F_0,
    \qquad
    \|\beta^\circ(\tau_{\rm oracle})-\beta^\star\|_2
    \leq
    C\sigma
    \sqrt{\frac{s+f_0+\log(2/\delta)}{n}}.
\]
\end{theorem}
Under noisy observations, ideal early stopping allows the limiting trajectory to recover all true coordinates before proposing any new null coordinate. 
At this stopping time, the estimation error has the least-squares scaling associated with the union $S^\star\cup S_{\rm init}$, with no ambient-dimensional logarithmic factor in the error bound.
In particular, when $f_0=0$, the stopped trajectory achieves exact support recovery and the estimation rate associated with knowing the true support. With inherited false positives, the recovered support can differ from $S^\star$ only through coordinates already present at initialization, and the error bound accounts for these additional coordinates through $f_0$.
As in Theorem~\ref{thm:noiseless-recovery}, the ambient-dimensional logarithmic term in the sample-size requirement involves $m=|R_0|$, counting both missing and incorrectly signed true coordinates. 
The additional terms $m^2+mf_0$ arise in our current analysis of the trajectory. We conjecture that these terms can be removed through a sharper analysis, while retaining the same recovery and estimation guarantees.

When $S_{\rm init}=\emptyset$, Theorem~\ref{thm:oracle-recovery} gives exact support recovery and squared $\ell_2$ error of order $\sigma^2s/n$ for the early-stopped S2S trajectory, provided $n\gtrsim s^2+s\log d$.
For comparison, \citet{vaskevicius2019implicit} establish looser estimation rates $\sigma^2s\log d/n$, improving to $\sigma^2s\log s/n$ at high signal-to-noise ratio, for early-stopped gradient descent under RIP.
For bounded signal condition number, their RIP assumption is ensured by $n\gtrsim s^2\log(ed/s)$ Gaussian samples. 
Our guarantees concern the limiting S2S trajectory under Gaussian design, whereas theirs apply to finite-step gradient descent on designs satisfying RIP.

{
\setlength{\intextsep}{1pt}
\setlength{\textfloatsep}{2pt}

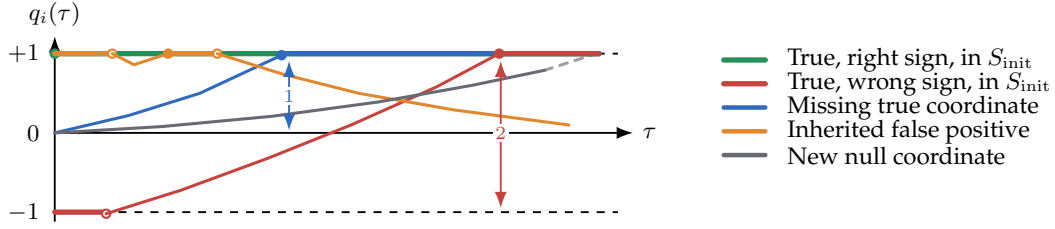
\begin{figure}[t]
    \centering
    \setlength{\abovecaptionskip}{1pt}
    \setlength{\belowcaptionskip}{0pt}

    \ificlr\vspace{-0.6em}\fi

\definecolor{goodgreen}{RGB}{36,145,90}
\definecolor{wrongred}{RGB}{198,67,63}
\definecolor{learnblue}{RGB}{49,105,190}
\definecolor{falseorange}{RGB}{225,135,45}
\definecolor{nullgray}{RGB}{105,105,115}


\begin{tikzpicture}[
    x=1cm,
    y=1.05cm,
    >=Latex,
    font=\small,
    trajectory/.style={
        line width=1.15pt,
        line cap=round,
        line join=round
    },
    active/.style={
        line width=2.0pt,
        line cap=round
    },
    event/.style={
        circle,
        draw,
        fill=white,
        inner sep=1.1pt,
        line width=0.8pt
    }
]


\draw[->,line width=0.7pt]
    (0,-1.12) -- (0,1.22)
    node[above] {$q_i(\tau)$};

\draw[->,line width=0.7pt]
    (0,0) -- (7.65,0)
    node[right] {$\tau$};

\draw[dashed,line width=0.7pt]
    (0,1) -- (7.45,1);

\draw[dashed,line width=0.7pt]
    (0,-1) -- (7.45,-1);

\node[anchor=east,font=\small] at (-0.08,1) {$+1$};
\node[anchor=east,font=\small] at (-0.08,0) {$0$};
\node[anchor=east,font=\small] at (-0.08,-1) {$-1$};


\draw[active,goodgreen]
    (0,1) -- (7.20,1);

\node[circle,fill=goodgreen,inner sep=1.4pt]
    at (0,1) {};


\draw[trajectory,learnblue]
    (0,0)
    -- (0.98,0.22)
    -- (1.92,0.50)
    -- (3.00,0.98);

\node[event,draw=learnblue,fill=learnblue]
    at (3.00,0.98) {};

\draw[active,learnblue]
    (3.00,1.0) -- (7.20,1.0);


\draw[active,wrongred]
    (0,-1) -- (0.68,-1);

\node[event,draw=wrongred] at (0.68,-1.02) {};

\draw[trajectory,wrongred]
    (0.68,-1.02)
    -- (1.68,-0.72)
    -- (2.82,-0.32)
    -- (3.92,0.10)
    -- (5.05,0.58)
    -- (5.88,1);

\node[event,draw=wrongred,fill=wrongred]
    at (5.88,1) {};

\draw[active,wrongred]
    (5.88,1) -- (7.20,1);


\draw[active,falseorange]
    (0,1) -- (0.76,1);

\node[event,draw=falseorange]
    at (0.76,1) {};

\draw[trajectory,falseorange]
    (0.76,1)
    -- (1.05,0.86)
    -- (1.50,1);

\node[event,draw=falseorange,fill=falseorange]
    at (1.50,1) {};

\draw[active,falseorange]
    (1.50,1) -- (2.15,1);

\node[event,draw=falseorange]
    at (2.15,1) {};

\draw[trajectory,falseorange]
    (2.15,1)
    -- (3.02,0.74)
    -- (4.02,0.50)
    -- (5.28,0.29)
    -- (6.80,0.10);


\draw[trajectory,nullgray]
    (0,0)
    -- (1.42,0.08)
    -- (2.88,0.21)
    -- (4.28,0.39)
    -- (5.52,0.60)
    -- (6.48,0.79);

\draw[trajectory,nullgray,dashed,opacity=0.65]
    (6.48,0.79) -- (7.10,0.98);


\draw[<->,wrongred,line width=0.7pt]
    (5.9,-0.95) -- (5.9,0.91);

\node[
    fill=white,
    inner sep=1pt,
    text=wrongred,
    font=\scriptsize
] at (5.9,0)
    {$2$};

\draw[<->,learnblue,line width=0.7pt]
    (3.10,0.03) -- (3.10,0.91);

\node[
    fill=white,
    inner sep=1pt,
    text=learnblue,
    font=\scriptsize
] at (3.10,0.46)
    {$1$};


\begin{scope}[shift={(8.85,0.92)}]

    \draw[active,goodgreen]
        (0,0.00) -- (0.55,0.00);
    \node[anchor=west,font=\small]
        at (0.72,0.00)
        {True, right sign, in $S_{\rm init}$};

    \draw[active,wrongred]
        (0,-0.30) -- (0.55,-0.30);
    \node[anchor=west,font=\small]
        at (0.72,-0.30)
        {True, wrong sign, in $S_{\rm init}$};

    \draw[trajectory,learnblue,line width=1.7pt]
        (0,-0.60) -- (0.55,-0.60);
    \node[anchor=west,font=\small]
        at (0.72,-0.60)
        {Missing true coordinate};

    \draw[trajectory,falseorange,line width=1.7pt]
        (0,-0.90) -- (0.55,-0.90);
    \node[anchor=west,font=\small]
        at (0.72,-0.90)
        {Inherited false positive};

    \draw[trajectory,nullgray,line width=1.7pt]
        (0,-1.20) -- (0.55,-1.20);
    \node[anchor=west,font=\small]
        at (0.72,-1.20)
        {New null coordinate};

\end{scope}

\end{tikzpicture}
    \ificlr\vspace{-0.5em}\fi

    \caption{Schematic dual trajectories under non-zero initialization. Missing and wrongly signed true coordinates travel dual distances $1$ and $2$, respectively, before entering with the correct sign.}
    \label{fig:algorithmic-behavior}
    \ificlr\vspace{-0.3em}\fi
\end{figure}
}
\begin{proof}[Proof sketch.]
Our proof relies on the saddle-to-saddle dynamics described in \Cref{sec:mirror-s2s}. 
Figure~\ref{fig:algorithmic-behavior} illustrates the three ingredients of the recovery argument.
First, the initial signed least-squares refit removes wrongly signed inherited true coordinates. 
Second, a uniform stability property ensures that correctly inherited true coordinates, as well as those subsequently recovered, remain active with the correct signs.
Third, the remaining true coordinates must reach their correct dual boundary: missing coordinates start at zero and must travel a dual distance one, whereas wrongly signed coordinates start at the opposite boundary and, after their initial removal, must travel a dual distance two. Concentration bounds for the projected gradients ensure that every remaining true coordinate progresses toward its correct boundary faster than any new null coordinate progresses toward either boundary. Consequently, all true coordinates are recovered before any new null coordinate is activated.

These bounds hold uniformly over the possible signed faces, so the argument remains valid despite intervening activations and deactivations of inherited false positives. In the balanced regime, this uniform control contributes the $m^2+mf_0$ terms to the sample-size requirement. We believe it is an artifact of the analysis.
The complete proof is given in Appendix~\ref{app:s2s-recovery}.
\end{proof}
\paragraph{Inherited false positives.} Unlike new null coordinates, inherited false positives start on the dual boundary. Even after deactivation, their distance to a subsequent boundary may be arbitrarily small, allowing them to reactivate before the remaining true coordinates have been learned. The recovery argument accommodates these events without requiring the removal of inherited false positives. Accordingly, the guarantee excludes new false positives but allows those inherited from initialization to remain. This cannot in general be removed by imposing a larger sample-size condition. 

Another way to see this is directly from the hyperbolic entropy limit established in \Cref{prop:implicit-bias}. Its leading term does not penalize inherited false positives when they retain the sign of their initialization. Consequently, the implicit regularization does not encourage these coefficients to vanish. This is in contrast to weighted Lasso, which continues to penalize inherited false positives, albeit with a smaller weight, thereby driving their coefficients toward zero.

\subsection{Recovery with the null-gradient stopping rule}
\label{subsec:stopping}

To get a computable stopping time with similar statistical properties, we introduce a data-dependent criterion based on the gradient of the next coordinate proposed by the saddle-to-saddle path of Algorithm~\ref{alg:saddle-to-saddle}. 
Our stopping rule builds on the residual-correlation criteria
of \citet{osher2016sparse} for sparse recovery via Bregman
inverse-scale-space dynamics. Here, we test the next proposed coordinate only when it lies outside $S_{\rm init}$, using a threshold calibrated uniformly over the possible trajectory faces. 
At each saddle of the trajectory, let $j_{k+1}$ denote the next proposed coordinate. We compare the magnitude of its gradient coordinate with a threshold $G_{\rm null}$, chosen as a uniform high-probability upper bound for null coordinates over the possible faces of the trajectory.

The sharp theoretical calibration uses the path-complexity quantity
$m+f_0$ and is given by \Cref{eq:Gnull} in Appendix~\ref{app:stopping}. The
same rule can be implemented using any known upper bound
$B\geq m+f_0$; see \Cref{rem:stopping}. For a new coordinate
$j_{k+1}$, the trajectory is continued if either
$j_{k+1}\in S_{\rm init}$ or
\[\textstyle
    |[\nabla L(\beta^{(k)})]_{j_{k+1}}|>G_{\rm null},
\]
and is stopped otherwise. We state the result for the sharp calibration $B=m+f_0$.
\begin{theorem}[Recovery with null-gradient early stopping]
\label{thm:main}
Consider  \Cref{ass:balanced-low-noise}. Denoting by $\tau_{\rm stop}$ the stopping time associated to the above stopping rule, which is fully described in Appendix~\ref{app:stopping}, there exists $C_L>0$, depending
only on $L$, such that, if
\[\textstyle
    n
    \ge
    C_L\left(
        s+f_0+m^2+mf_0
        +m\log\frac{d}{\delta}
    \right),
\]
then, with probability at least $1-\delta$, for some universal constant $C$, 
\[
    S^\star
    \subseteq
    \operatorname{supp}(\beta^\circ(\tau_{\rm  stop}))
    \subseteq
    S^\star\cup F_0,
    \qquad
    \|\beta^\circ(\tau_{\rm stop})-\beta^\star\|_2
    \leq
    C\sigma
    \sqrt{\frac{s+f_0+\log(2/\delta)}{n}}.
\]
\end{theorem}
Thus, with high probability, the null-gradient rule matches the oracle stopping time of \Cref{thm:oracle-recovery}: the remaining true coordinates
are still detected, whereas a newly proposed false
coordinate is not.

\begin{remark}[Unknown path complexity]\label{rem:stopping}
The preceding theorem uses the oracle calibration $B=m+f_0$ only to display the sharpest theoretical sample complexity. 
More generally, replacing $m+f_0$ by any known upper bound $B\geq m+f_0$ in the definition of $G_{\rm null}$ (\Cref{eq:Gnull}) yields the same recovery guarantee under the sample-size requirement $n\gtrsim s+f_0+mB+m\log\frac{d}{\delta}$. 
%
\end{remark}

\subsection{Comparison with Weighted Lasso }

We compare in this section the fine-tuning guarantees with the weighted-Lasso benchmark of Proposition~\ref{prop:wl-benchmark}, using the prescribed weight $\alpha_\star$.
Both methods exploit the pretrained support, but fine-tuning also depends on the inherited signs.
Write $m_{\rm miss}:=|S^\star\setminus S_{\rm init}|$ and  $s_{\rm init}:=|S_{\rm init}|$. Table~\ref{tab:comparison-wl} summarizes the noisy sample-size scalings under the respective signal-strength assumptions.
\begin{table}[t]
    \centering
    \small
    \begin{tabular}{@{}lcc@{}}
        \toprule
        & Early-stopped S2S & Weighted Lasso \\
        \midrule
        Sample size
        & $n\gtrsim s+f_0+m^2+mf_0+m\log(d)$
        & $n\gtrsim m_{\rm miss}\log(d)
           +(s-m_{\rm miss})\log(s_{\rm init})$ \\
        Support guarantee
        & $S^\star\subseteq S_{\rm final}
           \subseteq S^\star\cup F_0$
        & $\operatorname{supp}(\widehat\beta^{\rm WL})=S^\star$ \\
        \bottomrule
    \end{tabular}
    \caption{Noisy recovery guarantees under the respective assumptions: Theorem~\ref{thm:main} with $B=m+f_0$ for S2S, and Proposition~\ref{prop:wl-benchmark} with
$\sigma/\min_{i\in S^\star}|\beta_i^\star|=O(1)$ for weighted Lasso.
Numerical constants and confidence dependence are suppressed.
    }
    \label{tab:comparison-wl}
\end{table}

%

When $F_0=\emptyset$, both methods recover the true support under their respective assumptions, even if some inherited signs are incorrect. For a clean, correctly signed initialization, we additionally have $m=m_{\rm miss}=s-s_{\rm init}$, so both bounds restrict the ambient-dimensional logarithmic cost to the missing coordinates. Beyond the common term $m\log(d)$, the S2S bound contains
$s+m^2$, whereas the weighted-Lasso bound with the prescribed weight $\alpha_\star$ contains $(s-m)\log(s-m)$. Thus, S2S avoids the inherited logarithmic term but incurs the additional cost of controlling the trajectory.
When $m\lesssim\log(d)$, its sample-size requirement simplifies to $n\gtrsim s+m\log(d)$. Neither displayed sample-size scaling uniformly dominates the other.

With an imperfect initialization, weighted Lasso is unaffected by inherited sign errors, whereas fine-tuning counts wrongly signed true coordinates among the $m$ coordinates to be relearned. Inherited false positives enlarge $s_{\rm init}$ in the weighted-Lasso bound and contribute through $f_0$ and $mf_0$ in the S2S bound. The guarantees also differ: weighted Lasso recovers $S^\star$ exactly, whereas the noisy S2S guarantee allows inherited false positives to remain, $S^\star\subseteq \operatorname{supp}(\beta^\circ(\tau_{\rm  stop}))\subseteq S^\star\cup F_0$. This gap between S2S and weighted-Lasso is here mostly due to the fact that S2S does not penalize at all the pretrained support, and thus the inherited false positives.

Without noise, Theorem~\ref{thm:noiseless-recovery} gives exact recovery of the gradient-flow endpoint as $\mu\to0$ under $n\gtrsim s+f_0+m\log d+\log(1/\delta)$. Thus, the additional $m^2+mf_0$ terms are absent from the noiseless endpoint guarantee,\footnote{This term indeed appears when bounding uniformly over the whole trajectory of the weights. In the noiseless setting, we only have to control for the final point of the trajectory, so that no loose union bound is used.} which also ensures the removal of inherited false positives.

Finally, for a common recovered support, least-squares refitting produces the same estimator regardless of how that support was selected. The comparison therefore concerns the support recovered and the sample requirements for recovering it, rather than the estimation mechanism after selection.

\section{Conclusion}
We studied how a pretrained initialization can reduce the statistical cost of learning a sparse downstream task in DLNs.
The pretrained initialization changes both the implicit bias and training trajectory,
allowing inherited coordinates to avoid the ambient-dimensional support-identification cost.
In the noiseless setting, this yields exact recovery with a sample requirement depending on the coordinates that remain to be learned, while with noise, an appropriately stopped saddle-to-saddle trajectory recovers the remaining signal before introducing new false positives.
These results provide a simple setting in which the statistical benefit of fine-tuning can be characterized precisely.

\paragraph{Limitations and extensions.}
Our analysis assumes an independent Gaussian design, and the simplified noisy guarantees use a balanced low-noise regime. Extending the recovery analysis to correlated designs would require additional control of the projected gradients along the data-dependent trajectory. The pathwise guarantees concern the vanishing-imbalance limit; quantitative guarantees for finite layer imbalance $\mu>0$ remain to be established.

\subsection*{AI use statement}

In this work, we used generative AI tools to assist with polishing the writing, coding, and identifying relevant literature references. Generative AI tools were also used to explore some mathematical arguments. All mathematical proofs presented in the paper were developed, verified, and written by the human authors. Where AI tools provided useful suggestions, the authors critically evaluated, adapted, clarified, and improved them before incorporating the resulting arguments into the paper.

\ifarxiv
\subsubsection*{Acknowledgments}
This work was partially funded by the Swiss National Science Foundation, grant number 212111. 
This work benefited from the support of the FMJH Program PGMO. 
\fi

%



\bibliography{iclr2027_conference}
\bibliographystyle{iclr2027_conference}

\newpage

\appendix

\addcontentsline{toc}{section}{Appendix} 
\part{Appendix} 
\parttoc 

\section{Additional Results on the Saddle-to-Saddle Dynamics}
\label{app:dynamics}

This section provides the proofs of the dynamical results stated in
\Cref{sec:imp_bias}. We first derive the small-$\mu$ expansion of the Bregman
divergence associated with the non-zero initialization. We then record the
properties of the constrained saddles needed for the saddle-to-saddle
reduction.

\subsection{Implicit-bias expansion}

\begin{proof}[Proof of \Cref{prop:implicit-bias}]
Write
$\phi_{\mu,i}(x)
=\frac12\left[
x\operatorname{arcsinh}(x/\mu)-\sqrt{x^2+\mu^2}
\right]$
and recall that $\lambda_\mu=\frac12\log(1/\mu)$.
For every fixed $x\neq0$,
\[
\operatorname{arcsinh}\left(\frac{x}{\mu}\right)
=
\operatorname{sign}(x)
\left(
\log\frac1\mu+\log(2|x|)
\right)
+o(1).
\]
Consequently,
\[
\phi_{\mu,i}(x)
=
\lambda_\mu |x|
+
\frac12\bigl(|x|\log(2|x|)-|x|\bigr)
+o(1),
\]
and
\[
\phi'_{\mu,i}(x)
=
\operatorname{sign}(x)
\left(
\lambda_\mu+\frac12\log(2|x|)
\right)
+o(1).
\]

If $j\notin S_{\rm init}$, then $\beta_j^0=0$ and
$\phi'_{\mu,j}(0)=0$. Hence
\[
D_{\phi_\mu,j}(\beta_j,0)
=
\lambda_\mu|\beta_j|
+
\frac12
\bigl(
|\beta_j|\log(2|\beta_j|)-|\beta_j|
\bigr)
+o(1).
\]
In particular,
$D_{\phi_\mu,j}(\beta_j,0)/\lambda_\mu\to|\beta_j|$.

Consider now $i\in S_{\rm init}$ and set
$\varepsilon_i\coloneqq\operatorname{sign}(\beta_i^0)$ and
$a_i\coloneqq|\beta_i^0|$. Using the previous expansions in
\[
D_{\phi_\mu,i}(\beta_i,\beta_i^0)
=
\phi_{\mu,i}(\beta_i)
-
\phi_{\mu,i}(\beta_i^0)
-
\phi'_{\mu,i}(\beta_i^0)(\beta_i-\beta_i^0),
\]
we obtain
\[
\begin{aligned}
D_{\phi_\mu,i}(\beta_i,\beta_i^0)
&=
\lambda_\mu
\bigl(
|\beta_i|-\varepsilon_i\beta_i
\bigr) \\
&\quad+
\frac12
\Bigl(
|\beta_i|\log(2|\beta_i|)
-\varepsilon_i\beta_i\log(2a_i)
-|\beta_i|+a_i
\Bigr)
+o(1).
\end{aligned}
\]
If $\operatorname{sign}(\beta_i)=\varepsilon_i$, the singular term
vanishes and the finite correction reduces to
\[
\frac12
\left[
|\beta_i|
\log\frac{|\beta_i|}{|\beta_i^0|}
-
|\beta_i|
+
|\beta_i^0|
\right].
\]
If $\operatorname{sign}(\beta_i)=-\varepsilon_i$, then
$|\beta_i|-\varepsilon_i\beta_i=2|\beta_i|$, so changing the inherited
sign incurs a cost of order $\lambda_\mu|\beta_i|$.

Dividing by $\lambda_\mu$ and summing over the coordinates gives
\[
\frac{D_{\phi_\mu}(\beta,\beta^0)}{\lambda_\mu}
\longrightarrow
\sum_{j\notin S_{\rm init}}|\beta_j|
+
\sum_{i\in S_{\rm init}}
\left(
|\beta_i|
-
\operatorname{sign}(\beta_i^0)\beta_i
\right),
\]
which proves \Cref{prop:implicit-bias}.
\end{proof}

\subsection{Signed constrained saddles}

We next record the properties of the constrained minimizers used in the
limiting dynamics. Recall that, for $q\in[-1,1]^d$,
\[
F(q)
=
\{\beta\in\R^d:q\in\partial\|\beta\|_1\} \hspace{2em}
I(q) = \{ j : |q_{j}| = 1\}
\]
We assume both of the following conditions:
\begin{assumption}\label{ass:generalposition}
The design satisfies the general-position condition of
\citet{pesme2023saddle}: for any $r\leq \min(n,d)$, any distinct
$j_1,\ldots,j_r$ and signs
$\varepsilon_1,\ldots,\varepsilon_r\in\{-1,1\}$, the affine span of
$\varepsilon_1X_{j_1},\ldots,\varepsilon_rX_{j_r}$ contains no other
signed column $\pm X_j$.
\end{assumption}
\begin{assumption}\label{ass:maximal-support-length}
Moreover, every signed face visited by the saddle-to-saddle trajectory
satisfies $|I(q)|<n$.
\end{assumption}
\citet{pesme2023saddle} uses a different assumption to prove the uniqueness of the constrained saddles. In the zero-initialization setting, the limiting dual process starts from $q^{(0)}=0$ and remains in $\operatorname{span}(X^\top)$. Combined with their general-position assumption, this implies that the signed columns associated with the active dual face are linearly independent, and hence that the constrained minimizer is unique.

With a general non-zero initialization, the limiting dual state $q^{(0)}$ need not belong to $\operatorname{span}(X^\top)$, so this argument cannot be used directly. We instead impose \Cref{ass:maximal-support-length}. Since every visited face satisfies $|I(q)|<n$, the Gaussian design ensures that $X_{I(q)}$ has full column rank almost surely, which yields the uniqueness required in \Cref{lem:constrained-saddles}.

As in \citet[Proposition~1]{pesme2023saddle}, the saddles of the diagonal parametrization correspond to minimizers of $L$ restricted to a set of active coordinates.

\begin{lemma}[Constrained saddles]
\label{lem:constrained-saddles}
Assume \Cref{ass:maximal-support-length}. Under the Gaussian design, almost surely, for every signed face visited by the saddle-to-saddle trajectory,
\[
\beta(q)=\argmin_{\beta\in F(q)}L(\beta)
\]
is uniquely defined.
\end{lemma}

\begin{proof}
Fix a visited dual state $q$ and write $I=I(q)$. By \Cref{ass:maximal-support-length}, $|I|<n$. Since $X_I\in\R^{n\times |I|}$ has independent Gaussian columns, it has full column rank almost surely. As there are finitely many subsets of $\{1,\ldots,d\}$, this property holds simultaneously for every $I$ with $|I|<n$.

The restriction of $L$ to the coordinate subspace supported on $I$ has Hessian $X_I^\top X_I\succ0$. Hence it is strictly convex and coercive on this subspace, and therefore also strictly convex on the closed convex face $F(q)$. The constrained minimizer exists and is unique.
\end{proof}

\subsection{Convergence to the saddle-to-saddle dynamics}

We now justify the reduction of the continuous mirror flow to the saddle-to-saddle dynamics. The proof follows the convergence argument of \citet[Theorem~2 and Appendix~E]{pesme2023saddle}. The main difference is that the limiting dual process starts from the non-zero state $q^{(0)}$ induced by $\beta^0$.

\begin{theorem}[Convergence to the saddle-to-saddle dynamics]
\label{thm:s2s-convergence}
Let $\beta^\mu$ solve the mirror flow and set $\lambda_\mu=\frac12\log(1/\mu)$ and $q^\mu(\tau)=\nabla\phi_\mu(\beta^\mu(\lambda_\mu\tau))/\lambda_\mu$.
Assume \Cref{ass:generalposition,ass:maximal-support-length}, and let $(\beta^{(k)},q^{(k)},\tau_k)_k$ be the sequence generated by \Cref{alg:saddle-to-saddle}. Define the associated piecewise-constant trajectory by $\beta^\circ(\tau)=\beta^{(k)}$ for $\tau\in(\tau_k,\tau_{k+1})$.

Then, almost surely with respect to the Gaussian design, for every compact set $\mathcal K\subset(0,\infty)\setminus\{\tau_1,\tau_2,\ldots\}$,
\[
\sup_{\tau\in\mathcal K}
\left\|
\beta^\mu(\lambda_\mu\tau)-\beta^\circ(\tau)
\right\|_2
\longrightarrow 0
\qquad\text{as }\mu\downarrow0.
\]
\end{theorem}

\begin{proof}
We first identify the initial dual state. For $i\in S_{\rm init}$, the expansion of $\partial_i\phi_\mu(\beta_i^0)$ gives $q_i^\mu(0)\to\operatorname{sign}(\beta_i^0)$, while $q_i^\mu(0)=0$ for $i\notin S_{\rm init}$. Hence $q^\mu(0)\to q^{(0)}$.

The accelerated mirror equation is
\[
q^\mu(\tau)
=
q^\mu(0)
-
\int_0^\tau
\nabla L\bigl(\beta^\mu(\lambda_\mu s)\bigr)\,\mathrm ds.
\]
We now use the arc-length compactness argument of \citet[Appendix~E]{pesme2023saddle}. Set $\widetilde\beta^\mu(\tau)=\beta^\mu(\lambda_\mu\tau)$ and define $a_\mu(\tau)=\tau+\int_0^\tau\|\dot{\widetilde\beta}^\mu(s)\| _2\,ds$.
Writing $t_\mu=a_\mu^{-1}$ and $\widehat\beta^\mu(r)=\widetilde\beta^\mu(t_\mu(r))$, one has $\dot t_\mu+\|\dot{\widehat\beta}^\mu\|_2=1$.

The uniform path-length bound and the Arzelà--Ascoli argument used in \citet[Proposition~6 and Proposition~8]{pesme2023saddle} therefore give, up to extraction, local uniform convergence $(t_\mu,\widehat\beta^\mu)\to(t,\widehat\beta)$. Passing to the limit in the equation above, the corresponding dual limit satisfies
\[
q(r)
=
q^{(0)}
-
\int_0^r
\dot t(s)\nabla L(\widehat\beta(s))\,\mathrm ds.
\tag{A.2}
\]
Moreover, the small-$\mu$ limit of the normalized mirror map gives
$q(r)\in\partial\|\widehat\beta(r)\|_1$, and therefore
$\widehat\beta(r)\in F(q(r))$.

We next identify the extracted limit. This follows the induction used in \citet[Theorem~3]{pesme2023saddle}. During a saddle phase the primal variable is constant. By \Cref{lem:constrained-saddles}, the constrained minimizer on the current signed face is unique, so this constant value is necessarily $\beta^{(k)}=\argmin_{\beta\in F(q^{(k)})}L(\beta)$. On the corresponding accelerated-time interval, the previous equation reduces to
\[
q(\tau)
=
q^{(k)}
-
(\tau-\tau_k)\nabla L(\beta^{(k)}).
\]
The phase ends exactly when an inactive dual coordinate first reaches one of the thresholds $\pm1$. The resulting hitting time and face update are therefore those of \Cref{alg:saddle-to-saddle}.

It follows inductively that every subsequential limit visits the same constrained saddles, with the same dual evolution and the same event times as the saddle-to-saddle trajectory $\beta^\circ$. Since the constrained saddle associated with every visited face is unique, the limiting process is unique. Consequently all convergent subsequences have the same limit.

Finally, mapping the arc-length parametrization back to accelerated time as in the proof of \citet[Theorem~2, Appendix~E.1]{pesme2023saddle} yields $\beta^\mu(\lambda_\mu\tau)\to\beta^\circ(\tau)$ uniformly on every compact set that does not contain an event time.
\end{proof}

\subsection{Noiseless recovery}

We here prove \Cref{thm:noiseless-recovery}. 
\begin{proof}
First notice that as we are in the noiseless setting, the feasible interpolating set is characterized as follows:
\begin{equation*}
    \left\{\beta\in\R^d \mid \mathbf{X}\beta = \mathbf{y}\right\} = \beta^\star + \ker(\mathbf{X}).
\end{equation*}
Thus, for any interpolating solution $\beta$, we can write it as $\beta=\beta^\star + h$ with $h\in\ker(\mathbf{X})$. For any such $h$, we then have:
\begin{multline*}
    R_{\beta^0}(\beta^\star+h) - R_{\beta^0}(\beta^\star) = \sum_{i\in S^\star\setminus S_{\rm init}} \left(|\beta_i^\star+h_i|-|\beta_i^\star|\right) + \sum_{i\in(S_{\rm init}\cup S^\star)^c}|h_i| + \\
    \sum_{i\in S^\star\cap S_{\rm init}} \left(|\beta_i^\star+h_i|-|\beta_i^\star|-\sign(\beta_i^0)h_i\right) + \sum_{i\in S_{\rm init}\setminus S^\star}\left(|h_i|-\sign(\beta_i^0)h_i\right).
\end{multline*}
The first sum is lower bounded by $-\|h_{S^\star\setminus S_{\rm init}}\|_1$, the second by $\|h_{(S_{\rm init}\cup S^\star)^c}\|_1$ and the fourth by $0$. For the third sum, we have the identity $$|\beta_i^\star + h_i| = \big| |\beta_i^\star| + \sign(\beta_i^\star)h_i\big| \geq |\beta_i^\star| + \sign(\beta_i^\star)h_i,$$
so that
\begin{align}
     R_{\beta^0}(\beta^\star+h) - R_{\beta^0}(\beta^\star) &\geq \|h_{(S_{\rm init}\cup S^\star)^c}\|_1-\|h_{S^\star\setminus S_{\rm init}}\|_1 - 2 \|h_{W_0}\|_1\notag\\
     &\geq \|h_{(S_{\rm init}\cup S^\star)^c}\|_1-2\|h_{R_0}\|_1 \label{eq:Rbetanoiseless}
\end{align}
where $W_0 = \left\{i\in S_{\rm init}\cap S^\star \mid \sign(\beta_i^0)\neq \sign(\beta_i^\star)\right\}$, and $R_0=W_0 \cup (S^\star\setminus S_{\rm init})$ as defined in \Cref{sec:early-stopped}.

From there, we can apply Theorem 5 of \citet{mansour2017recovery} with, following their notations, $T=S^\star\cup S_{\rm init}$, $\widetilde T=(S^\star\cup S_{\rm init})\setminus R_0$, $A=\mathbf{X}$, $w=0$ and $C=1/3$, so that if
\begin{equation*}
    n \gtrsim |S^\star \cup S_{\rm init}| + |R_0|\ln(d) + \ln(1/\delta),
\end{equation*}
then with probability at least $1-\delta$, uniformly over all $h\in\ker(\mathbf{X})$,
\begin{equation}\label{eq:NPprop}
    \|h_{R_0}\|_1 \leq \frac{1}{3} \|h_{( S_{\rm init}\cup S^\star)^c}\|_1.
\end{equation}
In the following, we assume that \Cref{eq:NPprop} holds, since our specified sample complexity matches the one above ($|S^\star \cup S_{\rm init}|=s+f_0$ and $|R_0|=m$). 
\Cref{eq:Rbetanoiseless} then implies that for any $h\in\ker(\mathbf{X})$,
\begin{equation}\label{eq:Rboundnoiseless}
    R_{\beta^0}(\beta^\star+h) \geq R_{\beta^0}(\beta^\star) + \frac{1}{3}\|h_{(S_{\rm init}\cup S^\star)^c}\|_1.
\end{equation}
Moreover, since $\mathbf{X}\in\R^{n\times d}$ has values drawn as i.i.d. standard Gaussian, its distribution is invariant by rotation. In consequence, $\ker(\mathbf{X})$ is a $d-n$ subspace\footnote{We here assume that $n<d$, since the case $n\geq d$ directly yields that, almost surely, $\ker(\mathbf{X})=\{\mathbf{0}\}$.} of $\R^d$ selected uniformly at random. Hence, for any fixed subspace of dimension at most $n$, its intersection with $\ker(\mathbf{X})$ is almost surely trivial. 
In particular, since $s+f_0\leq n$, it holds almost surely that
\begin{equation*}
    \ker(\mathbf{X}) \cap \left\{ h\in \R^d \mid h_{(S_{\rm init}\cup S^\star)^c}= \mathbf{0}\right\} = \{\mathbf{0}\}.
\end{equation*}
\Cref{eq:Rboundnoiseless} then implies that $R_{\beta^0}(\beta^\star+h)>R_{\beta^0}(\beta^\star)$ for any $h\in\ker(\mathbf{X})\setminus\{\mathbf{0}\}$, i.e., $\beta^\star$ is the unique minimizer of the considered problem.

The second point is a direct consequence of the first one and \Cref{prop:implicit-bias}, since by definition $\beta_\infty^\mu = \argmin_{\beta\in\mathbb R^d: \mathbf X \beta = \mathbf y} D_{\phi_\mu}(\beta,\beta^0)$. So in particular, any limit point of $\beta_\infty^\mu$ as $\mu\to 0$ minimizes, among the interpolating solutions, the dominating term of $D_{\phi_\mu}(\beta,\beta^0)$ as $\mu\to0$, which is exactly given by $R_{\beta_0}$.
\end{proof}

\section{Statistical Preliminaries for Saddle-to-Saddle Recovery}
\label{app:stat-prelim}

We collect here the Gaussian estimates used in the recovery analysis. All
results are first stated on a fixed contaminated face. Uniform control over
the possible faces of the trajectory is postponed to
\Cref{app:s2s-recovery}.

Throughout this section, let $T\coloneqq(S^\star)^c$. For
$H\subseteq S^\star$ and $B\subseteq F_0$, let $\beta^{H,B}$ denote the
constrained saddle whose active support is $H\cup B$, with the correct signs
on $H$ and the inherited signs on $B$. We write
$r^{H,B}\coloneqq\mathbf y-\mathbf X\beta^{H,B}$ and
$g^{H,B}\coloneqq\mathbf X^\top r^{H,B}$.

\subsection{Contaminated faces and residual representation}

For a pair $(H,B)$, define the orthogonal projector
\[
P_{H\cup B}^{\perp}
\coloneqq
I-
\mathbf X_{H\cup B}
\left(
    \mathbf X_{H\cup B}^{\top}\mathbf X_{H\cup B}
\right)^{-1}
\mathbf X_{H\cup B}^{\top},
\]
and let $\nu_{H,B}\coloneqq n-|H|-|B|$ denote its rank.

\begin{lemma}[Residual representation on a contaminated face]
\label{lem:contaminated-residual-representation}
Let $H\subseteq S^\star$ and $B\subseteq F_0$, and assume that the saddle
$\beta^{H,B}$ lies in the relative interior of its signed face. Then
$\mathbf X_{H\cup B}^{\top}r^{H,B}=0$ and
\begin{equation}
\label{eq:contaminated-residual-representation}
r^{H,B}
=
P_{H\cup B}^{\perp}
\left(
    \mathbf X_{S^\star\setminus H}
    \beta^\star_{S^\star\setminus H}
    +
    \frac{\sigma}{\sqrt n}\varepsilon
\right).
\end{equation}
In particular, $r^{H,B}$ is measurable with respect to
$\sigma(\mathbf X_{S^\star\cup F_0},\varepsilon)$ and is independent of the
columns $(\mathbf X_j)_{j\in T\setminus F_0}$.
\end{lemma}

\begin{proof}
The relative-interior assumption gives the active normal equations
$\mathbf X_{H\cup B}^{\top}r^{H,B}=0$. Since $B\subseteq T$,
$\beta_B^\star=0$, while
\[
\mathbf y
=
\mathbf X_H\beta_H^\star
+
\mathbf X_{S^\star\setminus H}
\beta_{S^\star\setminus H}^\star
+
\frac{\sigma}{\sqrt n}\varepsilon.
\]
The residual is therefore the orthogonal projection of the last two terms
onto $\operatorname{span}(\mathbf X_{H\cup B})^\perp$, which gives
\Cref{eq:contaminated-residual-representation}. The measurability and
independence statements follow from the independence of the Gaussian
columns.
\end{proof}

\begin{lemma}[Residual norm on a contaminated face]
\label{lem:contaminated-residual-norm}
Under the assumptions of
\Cref{lem:contaminated-residual-representation}, there exists a universal
constant $C>0$ such that, for every $\eta\in(0,1)$, with probability at least
$1-\eta$,
\[
\|r^{H,B}\|_2
\le
C
\sqrt{\frac{\nu_{H,B}}{n}}
\left(
    \|\beta^\star_{S^\star\setminus H}\|_2+\sigma
\right),
\]
provided $\nu_{H,B}\ge C\log(2/\eta)$.
\end{lemma}

\begin{proof}
Conditionally on $\mathbf X_{H\cup B}$, the vector in
\Cref{eq:contaminated-residual-representation} is centered Gaussian with
covariance
\[
\frac{
    \|\beta^\star_{S^\star\setminus H}\|_2^2+\sigma^2
}{n}
P_{H\cup B}^{\perp}.
\]
Its norm is therefore a multiple of a chi-square norm in dimension
$\nu_{H,B}$. Gaussian norm concentration gives the stated bound.
\end{proof}

\subsection{Gradient bounds on a fixed contaminated face}

We first record the elementary conditional Gaussian estimate used to control
coordinates independent of a given residual.

\begin{lemma}[Conditional Gaussian maximum]
\label{lem:conditional-gaussian-bound}
Let $\mathcal F$ be a sigma-field, let $v\in\mathbb R^n$ be
$\mathcal F$-measurable, and let $J\subseteq[d]$ be nonempty. Assume that,
conditionally on $\mathcal F$, the columns $(\mathbf X_j)_{j\in J}$ are
independent with distribution
$\mathcal N(0,I_n/n)$. Then, for every $\eta\in(0,1)$, almost surely,
\[
\mathbb P\left(
    \left.
    \max_{j\in J}|\mathbf X_j^\top v|
    >
    \frac{\|v\|_2}{\sqrt n}
    \sqrt{2\log\left(\frac{2|J|}{\eta}\right)}
    \,\right|\,
    \mathcal F
\right)
\le \eta.
\]
Consequently, with probability at least $1-\eta$,
\[
\max_{j\in J}|\mathbf X_j^\top v|
\le
\frac{\|v\|_2}{\sqrt n}
\sqrt{2\log\left(\frac{2|J|}{\eta}\right)}.
\]
\end{lemma}

\begin{proof}
Conditionally on $\mathcal F$, the vector $v$ is fixed and, for every
$j\in J$,
\[
\mathbf X_j^\top v\mid\mathcal F
\sim
\mathcal N\left(0,\frac{\|v\|_2^2}{n}\right).
\]
Hence, for every $t>0$,
\[
\mathbb P\left(
    \left.
    |\mathbf X_j^\top v|>t
    \,\right|\,
    \mathcal F
\right)
\le
2\exp\left(
    -\frac{nt^2}{2\|v\|_2^2}
\right),
\]
with the result being immediate when $v=0$. A conditional union bound over
$j\in J$ gives
\[
\mathbb P\left(
    \left.
    \max_{j\in J}|\mathbf X_j^\top v|>t
    \,\right|\,
    \mathcal F
\right)
\le
2|J|
\exp\left(
    -\frac{nt^2}{2\|v\|_2^2}
\right).
\]
Taking
$t=\frac{\|v\|_2}{\sqrt n}
\sqrt{2\log(2|J|/\eta)}$
proves the conditional bound. Taking expectations removes the conditioning.
\end{proof}

The next estimate controls the gradient of one true coordinate which is not
yet active on the current face.

\begin{lemma}[One-coordinate true-gradient lower bound]
\label{lem:one-coordinate-true-gradient}
Under the assumptions of
\Cref{lem:contaminated-residual-representation}, fix
$i\in S^\star\setminus H$. There exists a universal constant $C>0$ such
that, for every $\eta\in(0,1)$, with probability at least $1-\eta$,
\[
\begin{aligned}
\operatorname{sign}(\beta_i^\star)g_i^{H,B}
\ge{}&
\frac{\nu_{H,B}}{2n}|\beta_i^\star|
\\
&-
C\frac{\sqrt{\nu_{H,B}}}{n}
\left(
    \|\beta^\star_{S^\star\setminus(H\cup\{i\})}\|_2
    +\sigma
\right)
\sqrt{\log\frac{6}{\eta}},
\end{aligned}
\]
provided $\nu_{H,B}\ge C\log(6/\eta)$.
\end{lemma}

\begin{proof}
Separating the contribution of $i$ in
\Cref{eq:contaminated-residual-representation} gives
\[
\begin{aligned}
\operatorname{sign}(\beta_i^\star)g_i^{H,B}
={}&
|\beta_i^\star|
\mathbf X_i^\top
P_{H\cup B}^{\perp}
\mathbf X_i
\\
&+
\operatorname{sign}(\beta_i^\star)
\mathbf X_i^\top
P_{H\cup B}^{\perp}
\left(
    \mathbf X_{S^\star\setminus(H\cup\{i\})}
    \beta^\star_{S^\star\setminus(H\cup\{i\})}
    +
    \frac{\sigma}{\sqrt n}\varepsilon
\right).
\end{aligned}
\]
We control the two terms separately. Conditionally on
$\mathbf X_{H\cup B}$,
\[
n\mathbf X_i^\top P_{H\cup B}^{\perp}\mathbf X_i
\sim
\chi^2_{\nu_{H,B}}.
\]
The lower-tail inequality of
\citet[Lemma~1]{laurentmassart2000adaptive} gives
\[
\mathbb P\left(
    \chi^2_{\nu_{H,B}}
    \le
    \nu_{H,B}-2\sqrt{\nu_{H,B}t}
\right)
\le e^{-t}.
\]
Taking $t=\log(6/\eta)$ shows that the quadratic term is at least
$\nu_{H,B}/(2n)$ under the displayed lower bound on $\nu_{H,B}$.

For the second term, apply
\Cref{lem:contaminated-residual-norm} to the partial residual obtained after
removing the contribution of $i$. Conditionally on this partial residual and
on $\mathbf X_{H\cup B}$, its scalar product with $\mathbf X_i$ is centered
Gaussian with variance equal to the squared partial-residual norm divided by
$n$. A Gaussian tail bound gives the second term in the claimed inequality.
A union bound over the chi-square event, the partial-residual norm event and
the scalar-product event concludes the proof.
\end{proof}

We finally control all null coordinates that were not inherited from the
initialization.

\begin{lemma}[Null gradients on a contaminated face]
\label{lem:fixed-face-null-gradients}
Under the assumptions of
\Cref{lem:contaminated-residual-representation}, there exists a universal
constant $C>0$ such that, for every $\eta\in(0,1)$, with probability at least
$1-\eta$,
\[
\max_{j\in T\setminus F_0}|g_j^{H,B}|
\le
C
\frac{\sqrt{\nu_{H,B}}}{n}
\left(
    \|\beta^\star_{S^\star\setminus H}\|_2+\sigma
\right)
\sqrt{\log\frac{2d}{\eta}},
\]
provided $\nu_{H,B}\ge C\log(2/\eta)$.
\end{lemma}

\begin{proof}
By \Cref{eq:contaminated-residual-representation}, $r^{H,B}$ is measurable
with respect to $\sigma(\mathbf X_{S^\star\cup F_0},\varepsilon)$.
Conditionally on this sigma-field, the columns
$(\mathbf X_j)_{j\in T\setminus F_0}$ remain independent Gaussian columns.
Applying \Cref{lem:conditional-gaussian-bound} with $v=r^{H,B}$ and
$J=T\setminus F_0$, and then using
\Cref{lem:contaminated-residual-norm} together with
$|T\setminus F_0|\le d$, gives the result.
\end{proof}

\section{Proof of Recovery under Ideal Stopping}
\label{app:s2s-recovery}

We now prove the recovery result of \Cref{thm:oracle-recovery}. We first
introduce the decomposition of the initialization used throughout the proof.

Let
\[
G_0
\coloneqq
\left\{
i\in S_{\rm init}\cap S^\star:
\operatorname{sign}(\beta_i^0)
=
\operatorname{sign}(\beta_i^\star)
\right\},
\]
and let
\[
W_0
\coloneqq
\left\{
i\in S_{\rm init}\cap S^\star:
\operatorname{sign}(\beta_i^0)
\neq
\operatorname{sign}(\beta_i^\star)
\right\}.
\]
We also write $M_0\coloneqq S^\star\setminus S_{\rm init}$, so that
$R_0=M_0\cup W_0$ and $m=|R_0|$. Recall that
$F_0=S_{\rm init}\setminus S^\star$ and $f_0=|F_0|$.

A coordinate in $M_0$ starts from the center of the dual interval and has to
travel distance one. A coordinate in $W_0$, once removed from the first
signed saddle, has to travel from the wrong boundary to the correct one and
therefore has to travel distance two. For $i\in R_0$, define $d_i=1$ if $i\in M_0$ and $d_i=2$ if $i\in W_0$.
For every nonempty $V\subseteq S^\star$, let
$\beta_{\min}(V)\coloneqq\min_{i\in V}|\beta_i^\star|$.

For $A\subseteq R_0$ and $B\subseteq F_0$, we denote by
$\beta^{A,B}$ a contaminated constrained saddle whose true active
coordinates are $G_0\cup A$, with their correct signs, and whose active
inherited false positives are $B$, with the signs they carry on the
considered signed face. The sign pattern on $B$ is left implicit in the
notation.

We write $r^{A,B}\coloneqq\mathbf y-\mathbf X\beta^{A,B}$ and $g^{A,B}\coloneqq\mathbf X^\top r^{A,B}$. All statements below involving $(A,B)$ are understood uniformly over the possible sign patterns of the active false-positive coordinates in $B$.

\subsection{From gradient separation to correct arrivals}

We first isolate the deterministic link between gradient separation and
dual hitting times.

\begin{lemma}[Effective-score separation implies hitting-time separation]
\label{lem:effective-score-hitting-times}
Consider the S2S path after the initial signed saddle, and ignore activation
and deactivation events involving coordinates in $F_0$. Suppose that at every contaminated saddle indexed by $A\subsetneq R_0$ and $B\subseteq F_0$, and for every sign pattern carried by the active false-positive coordinates in $B$,
\[
\operatorname{sign}(g_i^{A,B})
=
\operatorname{sign}(\beta_i^\star),
\qquad
i\in R_0\setminus A,
\]
and
\[
\min_{i\in R_0\setminus A}
\frac{|g_i^{A,B}|}{d_i}
>
\max_{j\in (S^\star)^c\setminus F_0}
|g_j^{A,B}|.
\]
Then the next effective arrival belongs to $R_0\setminus A$ and reaches the
dual boundary with the correct sign.
\end{lemma}

\begin{proof}
For $i\in R_0$, let $q_i^{\rm start}$ denote its dual value at the beginning
of the learning phase. Thus
$\operatorname{sign}(\beta_i^\star)q_i^{\rm start}=0$ for $i\in M_0$ and
$\operatorname{sign}(\beta_i^\star)q_i^{\rm start}=-1$ for $i\in W_0$.
Define its normalized dual progress by
\[
p_i(\tau)
\coloneqq
\frac{
\operatorname{sign}(\beta_i^\star)
\bigl(q_i(\tau)-q_i^{\rm start}\bigr)
}{d_i}.
\]
Before activation, $p_i$ starts from zero and reaches one exactly when
coordinate $i$ reaches the correct dual boundary.

For a new null coordinate
$j\in(S^\star)^c\setminus F_0$, one has $q_j(0)=0$. Between two consecutive
saddles,
\[
\dot p_i(\tau)
=
\frac{
\operatorname{sign}(\beta_i^\star)g_i^{A,B}
}{d_i}
=
\frac{|g_i^{A,B}|}{d_i},
\]
whereas, almost everywhere,
\[
\frac{\mathrm d}{\mathrm d\tau}|q_j(\tau)|
\le
|g_j^{A,B}|.
\]
The assumed separation therefore implies that every remaining true
coordinate makes normalized progress faster than any new null coordinate.

Events involving coordinates in $F_0$ do not reset the dual coordinates of
the other inactive variables. Hence the same comparison can be restarted
after each such event. Since all normalized progresses start from zero, a
new null coordinate cannot reach $|q_j|=1$ before one of the remaining true
coordinates reaches $p_i=1$. The sign assumption guarantees that this true
coordinate hits the correct boundary.
\end{proof}

\subsection{Uniform recovery conditions}

The recovery argument uses three properties of the contaminated path.

\paragraph{(C1) Initial purification.}
The first signed saddle removes every wrongly signed inherited true coordinate, i.e.,
$\beta_i^{(0)}=0$ for every $i\in W_0$.

\paragraph{(C2) Uniform stability.}
For every $A\subseteq R_0$, $B\subseteq F_0$, and every sign pattern on
the active coordinates in $B$, one has
$\beta_i^{A,B}\beta_i^\star>0$ for every $i\in G_0\cup A$.

\paragraph{(C3) Uniform detectable correct arrivals.}
For every $A\subsetneq R_0$, $B\subseteq F_0$, and every sign pattern on
the active coordinates in $B$,
$\operatorname{sign}(g_i^{A,B})=\operatorname{sign}(\beta_i^\star)$
for every $i\in R_0\setminus A$, and
\[
\min_{i\in R_0\setminus A}
\frac{|g_i^{A,B}|}{d_i}
>
\max_{j\in(S^\star)^c\setminus F_0}|g_j^{A,B}|.
\]

The following proposition verifies these three properties using the fixed-face
estimates of \Cref{app:stat-prelim}.

\begin{proposition}[Uniform recovery conditions]
\label{prop:gaussian-verification-recovery-conditions}
There exists a universal constant $C>0$ such that the following statements
hold.

\begin{enumerate}
    \item[(i)] Assume Condition (C2). If
    $W_0\neq\emptyset$, then Condition (C1) holds outside
    an additional event of probability at most $\delta$ provided
    \[
    n
    \ge
    s+f_0
    +
    C
    \frac{
        \bigl(\|\beta^\star_{R_0}\|_2+\sigma\bigr)^2
    }{
        \beta_{\min}(W_0)^2
    }
    \left(
        f_0+\log\frac{2m}{\delta}
    \right).
    \]

    \item[(ii)] Condition (C2) holds with probability at
    least $1-\delta$ provided
    \[
    n
    \ge
    s+f_0
    +
    C
    \left(
    1+
    \max_{\substack{
        A\subseteq R_0\\
        G_0\cup A\neq\emptyset
    }}
    \frac{
        \|\beta^\star_{R_0\setminus A}\|_2^2+\sigma^2
    }{
        \beta_{\min}(G_0\cup A)^2
    }
    \right)
    \left(
        m+f_0+\log\frac{s}{\delta}
    \right).
    \]

    \item[(iii)] Condition (C3) holds with probability at
    least $1-\delta$ provided
    \[
    n
    \ge
    s+f_0
    +
    C\left(1+
    \max_{A\subsetneq R_0}
    \frac{
        \bigl(
            \|\beta^\star_{R_0\setminus A}\|_2+\sigma
        \bigr)^2
    }{
        \beta_{\min}(R_0\setminus A)^2
    }\right)
    \left(
        m+f_0+\log\frac d\delta
    \right).
    \]
\end{enumerate}
The numerical cost of $d_i\in\{1,2\}$ is absorbed into $C$.
\end{proposition}

\begin{proof}
We prove the three statements separately. Numerical changes in the failure
probability are absorbed into the logarithms.
\paragraph{Uniform stability.}
Work on the probability-one event that
$X_{S^\star\cup F_0}$ has full column rank
($n\ge s+f_0$). For $A\subseteq R_0$ and
$D\subseteq F_0$, set
\[
H_A\coloneqq G_0\cup A,
\qquad
R_A\coloneqq R_0\setminus A,
\qquad
J\coloneqq H_A\cup D.
\]
For $H_A\neq\emptyset$, define the ordinary least-squares fit
\[
\widehat b^{A,D}
\coloneqq
(X_J^\top X_J)^{-1}X_J^\top y.
\]
Also set
\[
Q
\coloneqq
\max_{\substack{
    A\subseteq R_0\\
    H_A\neq\emptyset
}}
\frac{
    \|\beta^\star_{R_A}\|_2^2+\sigma^2
}{
    \beta_{\min}(H_A)^2
}.
\]
Under \cref{ass:balanced-low-noise}, $Q = Cm$ for each sample-size displayed afterward.
We first show that, simultaneously for all such $A,D$,
\[
\operatorname{sign}(\beta_i^\star)\widehat b_i^{A,D}>0,
\qquad i\in H_A.
\]
Since $S^\star=H_A\sqcup R_A$ and $\beta_D^\star=0$, write
\[
y=X_J\beta_J^\star+z_A,
\qquad
z_A\coloneqq
X_{R_A}\beta^\star_{R_A}
+\frac{\sigma}{\sqrt n}\varepsilon,
\qquad
v_A^2\coloneqq
\|\beta^\star_{R_A}\|_2^2+\sigma^2.
\]
For deterministic $A,D$, the vector $z_A$ is independent of $X_J$ and
has distribution $\mathcal N(0,v_A^2I_n/n)$. Consequently,
\[
\widehat b^{A,D}-\beta_J^\star\mid X_J
\sim
\mathcal N\!\left(
0,\frac{v_A^2}{n}(X_J^\top X_J)^{-1}
\right).
\]

For $i\in H_A$, let
\[
U_{J,i}
\coloneqq
nX_i^\top P_{J\setminus\{i\}}^\perp X_i,
\qquad
k_J\coloneqq n-|J|+1.
\]
The Schur-complement identity gives
\[
[(X_J^\top X_J)^{-1}]_{ii}
=
\frac{n}{U_{J,i}},
\qquad
U_{J,i}\sim\chi^2_{k_J}.
\]
If $v_A>0$, then conditionally on $X_J$,
\[
\Pr\!\left(
\operatorname{sign}(\beta_i^\star)
\widehat b_i^{A,D}\le0
\,\middle|\,X_J
\right)
\le
\exp\!\left(
-\frac{(\beta_i^\star)^2}{2v_A^2}U_{J,i}
\right).
\]
Taking expectations and using
$\mathbb E[e^{-t\chi_k^2}]=(1+2t)^{-k/2}$ gives
\begin{align*}
\Pr\!\left(
\operatorname{sign}(\beta_i^\star)
\widehat b_i^{A,D}\le0
\right)
&\le
\left(
1+\frac{(\beta_i^\star)^2}{v_A^2}
\right)^{-k_J/2}
\\
&\le
\exp\!\left(
-\frac{n-s-f_0+1}{2(1+Q)}
\right),
\end{align*}
where we used
$v_A^2/(\beta_i^\star)^2\le Q$ and
$\log(1+1/Q)\ge1/(1+Q)$ for $Q>0$.
If $v_A=0$, the fit is exact and the failure probability is zero.

There are at most $s2^{m+f_0}$ triples $(A,D,i)$. Hence, by a union bound,
the result holds with probability at least $1-\delta$ provided
\[
n
\ge
s+f_0
+
C(1+Q)
\left(
    m+f_0+\log\frac{s}{\delta}
\right).
\]
This is exactly the sample-size condition in~\textnormal{(ii)}.

Fix now $A\subseteq R_0$ with $H_A\neq\emptyset$,
$B\subseteq F_0$, and any prescribed signs
$\xi\in\{-1,+1\}^B$. Relax all sign constraints on the true block and let
\[
\widetilde b
\coloneqq
\arg\min_b
\left\{
L(b):
b_j=0\text{ for }j\notin H_A\cup B,
\ \xi_jb_j\ge0\text{ for }j\in B
\right\}.
\]
Set
\[
D\coloneqq\{j\in B:\widetilde b_j\neq0\}.
\]
The coordinates in $H_A$ are unconstrained, and for every $j\in D$ one has
$\xi_j\widetilde b_j>0$, so the sign constraints on $D$ are inactive.
Thus first-order optimality gives
\[
X_{H_A\cup D}^\top(y-X\widetilde b)=0.
\]
This identity holds even if some coordinates in $H_A$ are zero. Since
$\widetilde b_j=0$ for $j\notin H_A\cup D$, it follows that
\[
X_{H_A\cup D}^\top
\bigl(
y-X_{H_A\cup D}\widetilde b_{H_A\cup D}
\bigr)=0.
\]
As $X_{H_A\cup D}$ has full column rank,
\[
\widetilde b_{H_A\cup D}
=
(X_{H_A\cup D}^\top X_{H_A\cup D})^{-1}
X_{H_A\cup D}^\top y
=
\widehat b^{A,D}.
\]
On $\mathcal E_{\mathrm{OLS}}$, all coordinates in $H_A$ therefore have
their strictly correct signs. Hence $\widetilde b$ is feasible for the
original fully sign-constrained problem. Since it minimizes over the larger,
partially relaxed feasible set, it also minimizes over the original face.
Uniqueness gives $\widetilde b=\beta^{A,B}$ and proves Condition~(C2).

When $H_A=\emptyset$, the condition on this face is vacuous.
The argument holds for every sign pattern on $B$, without an additional
union bound over these signs.
\paragraph{Initial purification.}
Assume $W_0\neq\emptyset$. Consider the first signed least-squares problem
with the additional constraints $\beta_i=0$ for every $i\in W_0$.
Work on the stability event of condition (C2). Every coordinate in $G_0$ then remains active. Let $B\subseteq F_0$ be the false-positive set active at this reduced minimizer. Since this is the first signed saddle, these active false-positive coordinates still carry their inherited signs.

For $i\in W_0$, the remaining KKT condition for the full first-saddle problem
is
\[
\operatorname{sign}(\beta_i^0)g_i^{G_0,B}\le0.
\]
Since
$\operatorname{sign}(\beta_i^0)
=-\operatorname{sign}(\beta_i^\star)$,
it is enough to prove
$\operatorname{sign}(\beta_i^\star)g_i^{G_0,B}>0$.
For fixed $i$ and $B$, this follows from
\Cref{lem:one-coordinate-true-gradient}.

There are at most $m2^{f_0}$ possible pairs $(i,B)$. Taking
$\eta_0=\delta/(m2^{f_0})$, using
$\nu_{G_0,B}\ge n-s-f_0$, and bounding
\[
\|\beta^\star_{R_0\setminus\{i\}}\|_2
\le
\|\beta^\star_{R_0}\|_2,
\]
the required inequalities hold simultaneously provided
\[
n
\ge
s+f_0
+
C
\frac{
    \bigl(\|\beta^\star_{R_0}\|_2+\sigma\bigr)^2
}{
    \beta_{\min}(W_0)^2
}
\left(
    f_0+\log\frac{2m}{\delta}
\right).
\]
Therefore, on Condition~(C2), Condition~(C1) fails with additional
probability at most $\delta$, which proves statement~\textnormal{(i)}.

\paragraph{Uniform detectable arrivals.}
We are placing ourselves under the event described in condition (C2). Fix $A\subsetneq R_0$, $B\subseteq F_0$, and a
sign pattern on the active coordinates in $B$, and set
\[
H\coloneqq G_0\cup A,
\qquad
R\coloneqq R_0\setminus A,
\qquad
D\coloneqq\{j\in B:\beta_j^{A,B}\neq0\}.
\]
By Condition~(C2), every coordinate in $H$ is nonzero, while every
coordinate in $D$ is nonzero by definition. Moreover,
$\beta^{A,B}$ is feasible for the reduced signed face on $H\cup D$.
Since this reduced face is contained in the original one,
$\beta^{A,B}$ also minimizes $L$ over the reduced face. It therefore lies
in its relative interior. The residual and gradient are unchanged, so the
fixed-face estimates apply to this reduced face with $D$ in place of $B$.

For every $i\in R$, \Cref{lem:one-coordinate-true-gradient} gives
\[
\begin{aligned}
\operatorname{sign}(\beta_i^\star)g_i^{A,B}
\ge{}&
\frac{\nu_{H,D}}{2n}|\beta_i^\star|
\\
&-
C\frac{\sqrt{\nu_{H,D}}}{n}
\left(
    \|\beta^\star_{R\setminus\{i\}}\|_2+\sigma
\right)
\sqrt{\log\frac{6}{\eta_0}}.
\end{aligned}
\]
At the same saddle,
\Cref{lem:fixed-face-null-gradients} gives
\[
\max_{j\in(S^\star)^c\setminus F_0}
|g_j^{A,B}|
\le
C
\frac{\sqrt{\nu_{H,D}}}{n}
\left(
    \|\beta^\star_R\|_2+\sigma
\right)
\sqrt{\log\frac{2d}{\eta_0}}.
\]
Since $\nu_{H,D}\ge n-s-f_0$, the sample-size condition in
\textnormal{(iii)}, after increasing the universal constant $C$, implies
simultaneously
\[
\operatorname{sign}(g_i^{A,B})
=
\operatorname{sign}(\beta_i^\star),
\qquad i\in R,
\]
and
\[
\min_{i\in R}\frac{|g_i^{A,B}|}{d_i}
>
\max_{j\in(S^\star)^c\setminus F_0}|g_j^{A,B}|.
\]
Here we only use $d_i\le2$.

For each $B\subseteq F_0$, there are at most $2^{|B|}$ possible sign
patterns. Hence there are at most
\[
2^m\sum_{B\subseteq F_0}2^{|B|}
=
2^m3^{f_0}
\]
contaminated signed faces.
Taking the face-wise failure probability of order
$\delta/(2^m3^{f_0})$, and absorbing the additional union over the at most
$m$ remaining true coordinates into $\log(d/\delta)$, the previous
inequalities hold simultaneously provided
\[
n
\ge
s+f_0
+
C\left(
1+
\max_{A\subsetneq R_0}
\frac{
    \bigl(
        \|\beta^\star_{R_0\setminus A}\|_2+\sigma
    \bigr)^2
}{
    \beta_{\min}(R_0\setminus A)^2
}
\right)
\left(
    m+f_0+\log\frac d\delta
\right).
\]
Thus, outside an event of probability at most $\delta$, Condition~(C2)
implies Condition~(C3), which proves statement~\textnormal{(iii)}.
\end{proof}

\begin{remark}
The stability event is uniform over every signed configuration that may be visited before recovery. For a fixed active false-positive support $B\subseteq F_0$, there are $2^{|B|}$ possible sign patterns, and hence at most $2^m3^{f_0}$ contaminated signed faces in total. Initial purification is cheaper because, at the first signed saddle, active false positives still carry their inherited signs.
\end{remark}

\subsection{Recovery on the uniform event}

\begin{lemma}[Pathwise recovery on the uniform event]
\label{lem:pathwise-recovery}
Assume
Conditions (C1)--(C3).
Let $\tau_{\rm oracle}$ be the first saddle time at which every coordinate
of $S^\star$ is active with its correct sign. Then
$\tau_{\rm oracle}<\infty$ and
\[
S^\star
\subseteq
\operatorname{supp}(\beta^\circ(\tau_{\rm oracle}))
\subseteq
S^\star\cup F_0.
\]
In particular, $\tau_{\rm oracle}$ is reached before the ideal stopping rule
rejects any proposed coordinate outside $S^\star\cup S_{\rm init}$.
\end{lemma}

\begin{proof}
Condition (C1) removes every coordinate in $W_0$
at the first saddle, while
Condition (C2) keeps every coordinate in
$G_0$ active with its correct sign.

We argue by induction over the effective arrivals, namely the arrivals outside
the activation and deactivation cycles of $F_0$. Suppose that, after some
effective arrival, the set of learned true coordinates is $G_0\cup A$, with
$A\subseteq R_0$. The current inherited false-positive set may be any $B\subseteq F_0$,
with the sign pattern carried by the current signed face. If $A\neq R_0$,
Condition (C3) and \Cref{lem:effective-score-hitting-times}
imply that the next effective arrival belongs to $R_0\setminus A$ and enters
with the correct sign.
Condition (C2) then prevents its subsequent
removal. Thus $A$ grows by one at every effective true arrival.

The saddle-to-saddle process terminates after finitely many events under the
general-position assumptions. It cannot terminate while
$A\neq R_0$, since
Condition (C3) gives a nonzero gradient on a
remaining true coordinate and therefore a finite future hitting time.
Consequently, after at most $m$ effective true arrivals, $A=R_0$ and
$\tau_{\rm oracle}<\infty$.

Before $\tau_{\rm oracle}$, no coordinate in
$(S^\star)^c\setminus F_0$ can be the next effective arrival. At
$\tau_{\rm oracle}$, every true coordinate is active with its correct sign,
while the only possible additional active coordinates belong to $F_0$.
Hence
\[
S^\star
\subseteq
\operatorname{supp}(\beta^\circ(\tau_{\rm oracle}))
\subseteq
S^\star\cup F_0.
\]
\end{proof}

\subsection{Balanced low-noise regime}

Throughout this subsection, we work under \Cref{ass:balanced-low-noise}.

\begin{proposition}[Balanced recovery event]
\label{prop:balanced-recovery-event}
Assume \Cref{ass:balanced-low-noise}. There exists $C_L>0$, depending only
on $L$, such that, if
\[
n
\ge
C_L
\left(
    s+f_0+m^2+mf_0
    +m\log\frac d\delta
\right),
\]
then
Conditions (C1)--(C3)
hold simultaneously with probability at least $1-\delta$.
\end{proposition}

\begin{proof}
For every nonempty $V\subseteq S^\star$,
\[
\beta_{\min}(V)\ge a,
\qquad
\|\beta^\star_V\|_2\le La\sqrt{|V|},
\qquad
\sigma\le La\sqrt m.
\]
Therefore condition (C1) follows from
\[
n
\ge
C_L
\left(
    s+f_0+mf_0
    +m\log\frac{2m}{\delta}
\right).
\]
Moreover, for every $A\subsetneq R_0$,
\[
\frac{
    \bigl(
        \|\beta^\star_{R_0\setminus A}\|_2+\sigma
    \bigr)^2
}{
    \beta_{\min}(R_0\setminus A)^2
}
\le
C_Lm,
\]
and the same bound holds with
$\beta_{\min}(G_0\cup A)$ whenever $G_0\cup A\neq\emptyset$.
Hence condition (C2) follows from
\[
n
\ge
C_L
\left(
    s+f_0+m^2+mf_0
    +m\log\frac{s}{\delta}
\right),
\]
while condition (C3) follows from
\[
n
\ge
C_L
\left(
    s+f_0+m^2+mf_0
    +m\log\frac d\delta
\right).
\]
The last bound also implies the stability bound. Apply
\Cref{prop:gaussian-verification-recovery-conditions} with failure level
$\delta/3$ in each argument and take a union bound. The numerical factor
three is absorbed into the logarithms.
\end{proof}

\subsection{Estimation after recovery}

The support statement alone is not sufficient for
\Cref{thm:oracle-recovery}, which also gives an estimation bound. We use the
same least-squares argument for any recovered support between
$S^\star$ and $S^\star\cup F_0$.

\begin{lemma}[Least-squares estimation on the recovered support]
\label{lem:recovered-support-estimation}
Suppose
\[
S^\star
\subseteq
S_{\rm final}
\subseteq
S^\star\cup F_0,
\]
and let $\beta^{\rm alg}$ be the corresponding constrained saddle.
Then, provided
$n\ge C(s+f_0+\log(1/\delta))$, with probability at least $1-\delta$,
\[
\|\beta^{\rm alg}-\beta^\star\|_2
\le
C\sigma
\sqrt{
    \frac{
        s+f_0+\log(2/\delta)
    }{n}
}.
\]
\end{lemma}

\begin{proof}
Set $B\coloneqq S_{\rm final}\setminus S^\star\subseteq F_0$ and consider the Gram
event
\[
\mathcal E_{\rm Gram}
\coloneqq
\left\{
\frac12I
\preceq
\mathbf X_{S^\star\cup F_0}^{\top}
\mathbf X_{S^\star\cup F_0}
\preceq
2I
\right\}.
\]
It holds with probability at least $1-\delta/4$ provided
$n\ge C(s+f_0+\log(4/\delta))$.

The normal equations on $S^\star\cup B$ give
\[
\beta_B^{\rm alg}
=
\left(
\mathbf X_B^\top
P_{S^\star}^\perp
\mathbf X_B
\right)^{-1}
\mathbf X_B^\top
P_{S^\star}^\perp
\frac{\sigma}{\sqrt n}\varepsilon.
\]
On $\mathcal E_{\rm Gram}$, every corresponding Schur complement is bounded
below by $\frac12I$, hence
\[
\|\beta_B^{\rm alg}\|_2
\le
2
\left\|
\mathbf X_{F_0}^\top
P_{S^\star}^\perp
\frac{\sigma}{\sqrt n}\varepsilon
\right\|_2.
\]
This bound no longer depends on the random subset $B$. Conditionally on
$\mathbf X$, Gaussian norm concentration gives, with probability at least
$1-\delta/8$,
\[
\|\beta_B^{\rm alg}\|_2
\le
C\sigma
\sqrt{
    \frac{
        f_0+\log(8/\delta)
    }{n}
}.
\]

The active normal equations also give
\[
\begin{aligned}
\beta_{S^\star}^{\rm alg}-\beta_{S^\star}^\star
={}&
(\mathbf X_{S^\star}^\top\mathbf X_{S^\star})^{-1}
\mathbf X_{S^\star}^\top
\frac{\sigma}{\sqrt n}\varepsilon
\\
&-
(\mathbf X_{S^\star}^\top\mathbf X_{S^\star})^{-1}
\mathbf X_{S^\star}^\top
\mathbf X_B\beta_B^{\rm alg}.
\end{aligned}
\]
On $\mathcal E_{\rm Gram}$,
\[
\|\beta_{S^\star}^{\rm alg}-\beta_{S^\star}^\star\|_2
\le
C
\left\|
\mathbf X_{S^\star}^\top
\frac{\sigma}{\sqrt n}\varepsilon
\right\|_2
+
C\|\beta_B^{\rm alg}\|_2.
\]
With conditional probability at least $1-\delta/8$,
\[
\left\|
\mathbf X_{S^\star}^\top
\frac{\sigma}{\sqrt n}\varepsilon
\right\|_2
\le
C\sigma
\sqrt{
    \frac{
        s+\log(8/\delta)
    }{n}
}.
\]
Combining the previous estimates and taking a union bound yields
\[
\|\beta^{\rm alg}-\beta^\star\|_2
\le
C\sigma
\sqrt{
    \frac{
        s+f_0+\log(2/\delta)
    }{n}
}.
\]
\end{proof}

\subsection{Proof of Theorem~\ref{thm:oracle-recovery}}

\begin{proof}
Apply \Cref{prop:balanced-recovery-event} with confidence parameter
$\delta/2$. Under the sample-size assumption of the theorem,
Conditions (C1)--(C3)
therefore hold simultaneously with probability at least $1-\delta/2$.

On this event, \Cref{lem:pathwise-recovery} shows that the S2S trajectory
reaches a finite saddle time $\tau_{\rm oracle}$ before accepting any
coordinate outside $S^\star\cup S_{\rm init}$, and
\[
S^\star
\subseteq
\operatorname{supp}(\beta^\circ(\tau_{\rm oracle}))
\subseteq
S^\star\cup F_0.
\]

Applying \Cref{lem:recovered-support-estimation} with failure probability
$\delta/2$ gives
\[
\|\beta^\circ(\tau_{\rm oracle})-\beta^\star\|_2
\le
C\sigma
\sqrt{
    \frac{
        s+f_0+\log(2/\delta)
    }{n}
}.
\]
The additional Gram requirement in
\Cref{lem:recovered-support-estimation} is implied by the displayed
sample-size condition in the balanced regime after increasing $C_L$ if
necessary. A final union bound gives probability at least $1-\delta$.
\end{proof}

\section{Null-Gradient Early Stopping}
\label{app:stopping}

We now prove the validity of the data-dependent stopping rule used in
\Cref{thm:main}. The main difficulty is that the saddle visited by S2S is
data-dependent. A fixed-face Gaussian calibration is therefore not sufficient:
the null-gradient bound must hold simultaneously over the contaminated faces
that may be visited before and after recovery.

\subsection{Null-gradient calibration}

We first record the fixed-saddle calibration underlying the stopping rule.

\begin{lemma}[Null-gradient calibration at a fixed saddle]
\label{lem:fixed-saddle-null-calibration}
Fix a saddle-to-saddle step $k$. Let $\mathcal F_k$ be a sigma-field such
that $r^{(k)}$ and $U_k\subseteq[d]\setminus S_k$ are
$\mathcal F_k$-measurable. Assume that, conditionally on $\mathcal F_k$, the
columns $(\mathbf X_j)_{j\in U_k}$ are independent with distribution
$\mathcal N(0,I_n/n)$. Then, for every $\eta\in(0,1)$, almost surely,
\[
\mathbb P\left(
    \left.
    \max_{j\in U_k}|g_j^{(k)}|
    >
    \frac{\|r^{(k)}\|_2}{\sqrt n}
    \sqrt{2\log\left(\frac{2|U_k|}{\eta}\right)}
    \,\right|\,
    \mathcal F_k
\right)
\le\eta.
\]
The unconditional failure probability is also at most $\eta$.
\end{lemma}

\begin{proof}
By assumption, $r^{(k)}$ and $U_k$ are $\mathcal F_k$-measurable, while
conditionally on $\mathcal F_k$ the columns $(\mathbf X_j)_{j\in U_k}$ are
independent Gaussian vectors with law $\mathcal N(0,I_n/n)$. Apply
\Cref{lem:conditional-gaussian-bound} with
$\mathcal F=\mathcal F_k$, $v=r^{(k)}$ and $J=U_k$.
Since $g^{(k)}=\mathbf X^\top r^{(k)}$, this gives the stated threshold.
The unconditional failure probability is at most $\eta$ by taking
expectations.
\end{proof}

A fixed-saddle calibration cannot be applied directly after conditioning on a
trajectory selected using the null columns. We therefore uniformize the
previous bound over all contaminated signed faces. For each active
false-positive support $B\subseteq F_0$, all $2^{|B|}$ sign patterns must
be considered. Hence there are at most $2^m3^{f_0}$ such faces.
For a contaminated saddle $\beta^{A,B}$, define
\[
U_{A,B}^{\rm new}
\coloneqq
([d]\setminus S_{\rm init})
\setminus
\operatorname{supp}(\beta^{A,B}).
\]
We use the threshold
\begin{equation}\label{eq:Gnull}
G_{\rm null}(A,B;\eta)
\coloneqq
\frac{\|r^{A,B}\|_2}{\sqrt n}
\sqrt{
    2\log\left(
        \frac{
            2|U_{A,B}^{\rm new}|\,2^m3^{f_0}
        }{\eta}
    \right)
}.
\end{equation}
Since $|U_{A,B}^{\rm new}|\le d$,
\[
\log\left(
    \frac{
        2|U_{A,B}^{\rm new}|\,2^m3^{f_0}
    }{\eta}
\right)
\lesssim
m+f_0+\log\frac d\eta.
\]

\begin{remark}[Oracle calibration of the threshold]
\label{rem:oracle-threshold}
The factor $2^m3^{f_0}$ depends on the oracle quantities $m$ and $f_0$,
so the sharp theoretical threshold is not fully measurable from the
algorithmic history. It can be replaced by $3^B$ for any known upper
bound $B\ge m+f_0$, since
\[
2^m3^{f_0}\le3^{m+f_0}\le3^B.
\]
The resulting proof is unchanged, with the face-complexity term
$m+f_0$ replaced by $B$, up to universal constants.
\end{remark}

\subsection{Uniform separation around the null threshold}

The following lemma is the key event used by the stopping rule.

\begin{lemma}[Uniform separation around the null threshold]
\label{lem:uniform-threshold-separation}
Assume the sample-size requirement of
\Cref{prop:gaussian-verification-recovery-conditions}\textnormal{(iii)}, with
$\log(d/\delta)$ replaced by $\log(d/(\delta\eta))$. Then, with probability
at least $1-\delta-\eta$, simultaneously for every $A\subseteq R_0$, $B\subseteq F_0$, and every sign pattern on the active coordinates in $B$,
\[
\max_{j\in U_{A,B}^{\rm new}\cap(S^\star)^c}
|g_j^{A,B}|
\le
G_{\rm null}(A,B;\eta).
\]
Moreover, if $A\subsetneq R_0$, then
\[
G_{\rm null}(A,B;\eta)
<
\min_{i\in R_0\setminus A}
\frac{|g_i^{A,B}|}{d_i}.
\]
Hence, before recovery, every genuinely new true proposal lies above the
null threshold, whereas after recovery every genuinely new null proposal
lies below it.
\end{lemma}

\begin{proof}
We construct two uniform events, one for the true-gradient lower bound
and one for the residual-norm upper bound, each with failure probability
at most $\delta/2$.

Work on Condition~(C2). Fix a saddle before full recovery, indexed by
$A\subsetneq R_0$, $B\subseteq F_0$, and the prescribed sign pattern on
$B$, and set
\[
    H\coloneqq G_0\cup A,
    \qquad
    R\coloneqq R_0\setminus A,
    \qquad
    D\coloneqq
    \{j\in B:\beta_j^{A,B}\neq0\}.
\]
By Condition~(C2), every coordinate in $H$ is nonzero, while every
coordinate in $D$ is nonzero by definition. Moreover,
$\beta^{A,B}$ also minimizes $L$ over the reduced signed face on
$H\cup D$, and lies in its relative interior. Hence the fixed-face
estimates apply on this reduced face. Set
\[
    \nu
    \coloneqq
    n-|H|-|D|
    =
    n-|G_0|-|A|-|D|.
\]

Using the true-gradient estimate uniformly over the deterministic reduced
signed faces, with total failure probability at most $\delta/2$, and
absorbing $d_i\in\{1,2\}$ into the universal constants, we obtain
\[
    \min_{i\in R}
    \frac{|g_i^{A,B}|}{d_i}
    \ge
    c\frac{\nu}{n}\beta_{\min}(R).
\]
Here the residual and gradient at the original saddle are unchanged when
the zero coordinates in $B\setminus D$ are removed from the face.

For the same deterministic reduced face, the projected residual is
\[
    \rho^{H,D}
    =
    P_{H\cup D}^\perp
    \left(
        X_R\beta_R^\star
        +\frac{\sigma}{\sqrt n}\varepsilon
    \right).
\]
Gaussian norm concentration, followed by the same uniformization over
the deterministic reduced signed faces, gives with total failure
probability at most $\delta/2$,
\[
    \|r^{A,B}\|_2
    =
    \|\rho^{H,D}\|_2
    \le
    C
    \sqrt{\frac{\nu}{n}}
    \left(
        \|\beta_R^\star\|_2+\sigma
    \right).
\]
Thus, on the intersection of these two events, which has probability at
least $1-\delta$, both estimates hold simultaneously at every saddle for
which Condition~(C2) holds.

Since $|U_k^{\rm new}|\le d$, we therefore have
\[
\begin{aligned}
    G_{null} (k,\eta_k;U_k^{\rm new})
    &=
    \frac{\|r^{A,B}\|_2}{\sqrt n}
    \sqrt{
        2\log\!\left(
            \frac{2|U_k^{\rm new}|}{\eta_k}
        \right)
    }
\\
    &\le
    C
    \frac{\sqrt{\nu}}{n}
    \left(
        \|\beta_R^\star\|_2+\sigma
    \right)
    \sqrt{\log\frac{2d}{\eta_k}}.
\end{aligned}
\]
Since $\eta_k\ge\eta_{\min}$, the strengthened sample-size assumption,
with
\[
    \log\frac d\delta
    \quad\text{replaced by}\quad
    \log\frac{d}{\delta\eta_{\min}},
\]
implies, after increasing the universal constant,
\[
    c\frac{\nu}{n}\beta_{\min}(R)
    >
    C
    \frac{\sqrt{\nu}}{n}
    \left(
        \|\beta_R^\star\|_2+\sigma
    \right)
    \sqrt{\log\frac{2d}{\eta_k}}.
\]
Hence
\[
    \min_{i\in R_0\setminus A}
    \frac{|g_i^{A,B}|}{d_i}
    >
    G_{null}(k,\eta_k;U_k^{\rm new}).
\]

Finally, a genuinely new true proposal belongs to $M_0$, and therefore
has $d_{j_{k+1}}=1$. Consequently,
\[
    |g_{j_{k+1}}^{(k)}|
    >
    G_{null}(k,\eta_k;U_k^{\rm new}),
\]
which proves the claim.
\end{proof}

\subsection{The non-stopping rule}

We now give the precise stopping rule used in \Cref{thm:main}. At saddle
$\beta^{(k)}$, let $j_{k+1}$ be the coordinate proposed by the next
saddle-to-saddle hitting time and set
\[
U_k^{\rm new}
\coloneqq
([d]\setminus S_{\rm init})\setminus S_k.
\]
The threshold evaluated at the current saddle is
\[
G_{\rm null}(k,\eta)
\coloneqq
\frac{\|r^{(k)}\|_2}{\sqrt n}
\sqrt{
    2\log\left(
        \frac{
            2|U_k^{\rm new}|\,2^m3^{f_0}
        }{\eta}
    \right)
}.
\]

Coordinates inherited from the initialization are declared non-stopping. More
precisely, if $j_{k+1}\in S_{\rm init}$, the trajectory is continued without
testing its gradient. If $j_{k+1}\notin S_{\rm init}$, the proposal is
accepted when
\[
|g_{j_{k+1}}^{(k)}|
>
G_{\rm null}(k,\eta),
\]
and the trajectory is stopped otherwise.

The exclusion of $S_{\rm init}$ from the stopping test is necessary because
an inherited coordinate may be deactivated while its dual variable remains
arbitrarily close to the boundary. Its next hitting time
\[
\Delta\tau_{k,j}
=
\frac{
    1-\operatorname{sign}(g_j^{(k)})q_j^{(k)}
}{
    |g_j^{(k)}|
}
\]
can then be arbitrarily small even when $|g_j^{(k)}|$ is compatible with
noise. Such activation--deactivation cycles should therefore not trigger
stopping.

Since the event of
\Cref{lem:uniform-threshold-separation} is simultaneous over all contaminated
faces, no additional union bound over tests, visited saddles, or repeated
visits to the same face is needed.

\begin{proposition}[Validity of the null-gradient stopping rule]
\label{prop:null-gradient-stopping-validity}
Assume
Conditions (C1)--(C3)
and the sample-size requirement of
\Cref{lem:uniform-threshold-separation}. Then, with probability at least
$1-\delta-\eta$, the stopping rule reaches full recovery and rejects the
first genuinely new null proposal. Consequently,
\[
S^\star
\subseteq
S_{\rm final}
\subseteq
S^\star\cup F_0.
\]
\end{proposition}

\begin{proof}
By
Condition (C3) and \Cref{lem:effective-score-hitting-times},
every effective arrival before full recovery belongs to $R_0$ and enters
with the correct sign. Work on the event of
\Cref{lem:uniform-threshold-separation}.

A proposal in $W_0\subseteq S_{\rm init}$ is non-stopping by definition.
A genuinely new true proposal therefore belongs to $M_0$, and the second
inequality of \Cref{lem:uniform-threshold-separation} gives
\[
|g_{j_{k+1}}^{(k)}|
>
G_{\rm null}(k,\eta).
\]
Thus the rule cannot stop before full recovery.

After recovery, the first genuinely new null proposal is evaluated on a face $(R_0,B)$ for some $B\subseteq F_0$. The first inequality of
\Cref{lem:uniform-threshold-separation} gives
\[
|g_{j_{k+1}}^{(k)}|
\le
G_{\rm null}(k,\eta),
\]
so this proposal is rejected. At that time every true coordinate is active,
and any additional active coordinate belongs to $F_0$, which proves the
support inclusion.
\end{proof}

\subsection{Proof of Theorem~\ref{thm:main}}

\begin{proof}
Assume \Cref{ass:balanced-low-noise}. Apply
\Cref{prop:balanced-recovery-event} with failure probability $\delta/3$.
Under the sample-size condition of the theorem, the three recovery conditions
therefore hold simultaneously with probability at least $1-\delta/3$.

Next apply \Cref{lem:uniform-threshold-separation} with recovery failure
level $\delta/3$ and null-threshold failure level $\eta=\delta/3$. Since
\[
\log\left(
    \frac{
        2d\,2^m3^{f_0}
    }{\eta}
\right)
\lesssim
m+f_0+\log\frac d\delta,
\]
and the true-gradient comparison already carries a face-wise confidence
factor, the required sample size is, after increasing $C_L$,
\[
n
\ge
C_L\left(
    s+f_0+m^2+mf_0
    +m\log\frac{d}{\delta^2}
\right).
\]
On the intersection of these events,
\Cref{prop:null-gradient-stopping-validity} yields
\[
S^\star
\subseteq
\operatorname{supp}(\beta^\circ(\tau_{\rm stop}))
\subseteq
S^\star\cup F_0.
\]

Finally, apply \Cref{lem:recovered-support-estimation} with failure
probability $\delta/3$. It gives
\[
\|\beta^\circ(\tau_{\rm stop})-\beta^\star\|_2
\le
C\sigma
\sqrt{
    \frac{
        s+f_0+\log(2/\delta)
    }{n}
}.
\]
The Gram requirement of
\Cref{lem:recovered-support-estimation} is implied by the sample-size
condition above after increasing $C_L$ if necessary. A union bound over the
three events gives the result with probability at least $1-\delta$.
\end{proof}

\subsection{Using an upper bound on the path complexity}

For completeness, the stopping rule does not require exact knowledge of
$m+f_0$. If $B\ge m+f_0$ is any known upper bound, one may replace
$2^m3^{f_0}$ by $3^B$ in $G_{\rm null}$. The proof of
\Cref{lem:uniform-threshold-separation} is unchanged since
\[
2^m3^{f_0}\le3^{m+f_0}\le3^B.
\]
In the balanced regime, this only changes universal constants and gives,
up to constants,
\[
n
\gtrsim
s+f_0+mB+m\log\frac d\delta.
\]
The universal choice $B=d$ is fully observable but leads to the
conservative requirement $n\gtrsim md$.
\section{Proof of the Weighted-Lasso Benchmark}
\label{app:wl-proofs}

We write
$m_{\rm miss}:=|S^\star\setminus S_{\rm init}|$,
$a_0:=|S^\star\cap S_{\rm init}|$ and
$s_\alpha:=m_{\rm miss}+a_0/\alpha^2$.
Let $z:=X^\top(\sigma\varepsilon/\sqrt n)$ and
\[
L_{\alpha,\delta}
:=
\max\left\{
    \alpha\sqrt{\log\frac{4|S_{\rm init}|}{\delta}},
    \sqrt{\log\frac{4|S_{\rm init}^c|}{\delta}}
\right\}.
\]
We also write for any $\beta\in\R^d$, $\|\beta\|_{1,\alpha}\coloneqq \|\beta_{S^c_{\rm init}} \|_1 + \frac{1}{\alpha} \| \beta_{S_{\rm init}} \|_1 $.

\subsection{Weighted geometry and estimation bounds}

Define
$\mathcal C_\alpha(S^\star):=
\{v:\|v_{(S^\star)^c}\|_{1,\alpha}
\le3\|v_{S^\star}\|_{1,\alpha}\}$ and
\[
\phi_\alpha^2
:=
\inf_{v\in\mathcal C_\alpha(S^\star)\setminus\{0\}}
\frac{s_\alpha\|Xv\|_2^2}
     {\|v_{S^\star}\|_{1,\alpha}^2},
\qquad
\kappa_\alpha^2
:=
\inf_{v\in\mathcal C_\alpha(S^\star)\setminus\{0\}}
\frac{\|Xv\|_2^2}{\|v\|_2^2}.
\]

We first record a sufficient Gaussian condition ensuring that these
two quantities are bounded away from zero.

\begin{lemma}[Weighted geometry]
\label{lem:weighted-geometry}
There exists a universal constant $C>0$ such that, for every
$\eta\in(0,1)$, if
\[
n
\ge
C\left[s + s_\alpha \max\left(\alpha^2\log(\max(1,|S_{\rm init}\setminus S^\star|)), \log(\max(1,|(S_{\rm init}\cup S^\star)^c|))\right)
    +
    \log\frac1\eta
\right],
\]
then, with probability at least $1-\eta$,
$\kappa_\alpha^2\ge1/4$ and $\phi_\alpha^2\ge1/4$.
\end{lemma}
\begin{proof}
We apply Theorem~3.4 and Lemma~3.5 of
\citet{bah2016sample} to the weighted cone
$\mathcal C_\alpha(S^\star)$.

Multiplying all the weights defining $\|\cdot\|_{1,\alpha}$ by
$\alpha$ does not change the cone. Introduce the auxiliary weights
\[
w_j :=
\begin{cases}
3, & j\in S^\star\cap S_{\rm init},\\
3\alpha, & j\in S^\star\setminus S_{\rm init},\\
1, & j\in S_{\rm init}\setminus S^\star,\\
\alpha, & j\in (S_{\rm init}\cup S^\star)^c.
\end{cases}
\]
Then $w_j\ge1$ and
\[
v\in\mathcal C_\alpha(S^\star)
\quad\Longleftrightarrow\quad
\|v_{(S^\star)^c}\|_{w,1}
\le
\|v_{S^\star}\|_{w,1},
\]
where $\|\cdot\|_{w,1}$ follows the definition of \citet{bah2016sample}.

Hence $\mathcal C_\alpha(S^\star)$ is equal to the cone
$Z_{w,S^\star}$ of \citet[Lemma~3.5]{bah2016sample}.
Moreover,
\[
w(S^\star)
=
\sum_{j\in S^\star}w_j^2
=
9\left(a_0+\alpha^2m_{\rm miss}\right)
=
9\alpha^2s_\alpha .
\]

Lemma~3.5 of \citet{bah2016sample} therefore gives
\begin{align*}
\omega\!\left(
\mathcal C_\alpha(S^\star)\cap\mathbb S^{d-1}
\right)
\le {}&
\sqrt{s}
+
\inf_{h\ge0}
\Bigg[
3h\alpha\sqrt{s_\alpha}
\\
&\quad+
\left(\frac{2}{\pi e}\right)^{1/4}
\left(
\frac{e^{-h^2/2}}{h^2}
|S_{\rm init}\setminus S^\star|
+
\frac{e^{-\alpha^2h^2/2}}{\alpha^2h^2}
|(S_{\rm init}\cup S^\star)^c|
\right)^{1/2}
\Bigg],
\end{align*}
where $\omega\!\left(
\mathcal C_\alpha(S^\star)\cap\mathbb S^{d-1}
\right)$ is the Gaussian width of $\mathcal C_\alpha(S^\star)$.

Let
\[
A:=|S_{\rm init}\setminus S^\star|,
\qquad
B:=|(S_{\rm init}\cup S^\star)^c|.
\]
By Lemma~3.5 of \citet{bah2016sample},
\[
\omega\!\left(
\mathcal C_\alpha(S^\star)\cap\mathbb S^{d-1}
\right)
\le
\sqrt{s}
+
\inf_{h>0}
\left[
3h\alpha\sqrt{s_\alpha}
+
C
\left(
\frac{A e^{-h^2/2}}{h^2}
+
\frac{B e^{-\alpha^2h^2/2}}{\alpha^2h^2}
\right)^{1/2}
\right].
\]

Set
\[
M:=
\max\left\{
\log(1+A),
\frac{1}{\alpha^2}\log (1+eB)
\right\},
\]
with the corresponding term omitted when one of the sets is empty, and take
\[
h^2=cM
\]
for a sufficiently large universal constant $c$.

The first term satisfies
\[
3h\alpha\sqrt{s_\alpha}
=
3\sqrt{c}\,\alpha\sqrt{s_\alpha M}
\lesssim
\sqrt{
s_\alpha
\max\left\{
\alpha^2\log(1+A),
\log(1+B)
\right\}
}.
\]

Moreover, since $M\ge \log(1+A)$,
\[
A e^{-h^2/2}
\le
A e^{-c\log(eA)/2}
=
e^{-c/2}A^{1-c/2}.
\]
Thus, for $c\ge4$,
\[
\frac{A e^{-h^2/2}}{h^2}
\lesssim 1.
\]
Similarly, since $\alpha^2M\ge\log(1+B)$,
\[
B e^{-\alpha^2h^2/2}
\le
B e^{-c\log(eB)/2}
=
e^{-c/2}B^{1-c/2},
\]
and therefore
\[
\frac{B e^{-\alpha^2h^2/2}}{\alpha^2h^2}
\lesssim1.
\]

Hence
\[
\left(
\frac{A e^{-h^2/2}}{h^2}
+
\frac{B e^{-\alpha^2h^2/2}}{\alpha^2h^2}
\right)^{1/2}
\lesssim1.
\]
We conclude that
\[
\omega\!\left(
\mathcal C_\alpha(S^\star)\cap\mathbb S^{d-1}
\right)
\lesssim
\sqrt{s}
+
\sqrt{
s_\alpha
\max\left\{
\alpha^2\log(1+A),
\log(1+B)
\right\}
}.
\]
Using $(a+b+c)^2\lesssim a^2+b^2+c^2$ and absorbing the constant term
into $s$, we obtain
\[
\omega^2\!\left(
\mathcal C_\alpha(S^\star)\cap\mathbb S^{d-1}
\right)
\lesssim
s
+
s_\alpha
\max\left\{
\alpha^2\log(eA),
\log(eB)
\right\}.
\]
Equivalently,
\[
\omega^2\!\left(
\mathcal C_\alpha(S^\star)\cap\mathbb S^{d-1}
\right)
\lesssim
s
+
s_\alpha
\max\left\{
\alpha^2\log\!\bigl(e|S_{\rm init}\setminus S^\star|\bigr),
\log\!\bigl(e|(S_{\rm init}\cup S^\star)^c|\bigr)
\right\}.
\]
Now write $X=A/\sqrt n$, where $A$ has i.i.d. standard Gaussian
entries, and set
\[
T:=\mathcal C_\alpha(S^\star)\cap\mathbb S^{d-1}.
\]
By Gordon's escape-through-a-mesh theorem,
Theorem~3.4 of \citet{bah2016sample}, for every $t>0$,
with probability at least $1-e^{-t^2/2}$,
\[
\inf_{v\in T}\|Av\|_2
\ge
\mathbb E\|g\|_2-\omega(T)-t,
\]
where $g\sim\mathcal N(0,I_n)$. Taking
$t=\sqrt{2\log(1/\eta)}$ therefore gives, with probability at least
$1-\eta$,
\[
\inf_{v\in T}\|Xv\|_2
\ge
\frac{
\mathbb E\|g\|_2
-
\omega(T)
-
\sqrt{2\log(1/\eta)}
}{\sqrt n}.
\]

From the Gaussian-width estimate above, 
\[
\omega(T)^2
\lesssim
s+
s_\alpha
\max\left\{
\alpha^2\log\!\bigl(e|S_{\rm init}\setminus S^\star|\bigr),
\log\!\bigl(e|(S_{\rm init}\cup S^\star)^c|\bigr)
\right\}.
\]
Hence, under the stated sample-size condition,
\[
\omega(T)
+
\sqrt{2\log(1/\eta)}
\le
c\sqrt n
\]
for a sufficiently small universal constant $c>0$, after increasing
the constant in the sample-size bound.

Moreover, for $g\sim\mathcal N(0,I_n)$,
$\mathbb E\|g\|_2\asymp\sqrt n$; in particular,
$\mathbb E\|g\|_2\ge c_0\sqrt n$ for a universal constant $c_0>1/2$.
Choosing the universal constant in the sample-size condition large
enough therefore yields
\[
\mathbb E\|g\|_2
-
\omega(T)
-
\sqrt{2\log(1/\eta)}
\ge
\frac12\sqrt n.
\]
Consequently,
\[
\inf_{v\in T}\|Xv\|_2
\ge
\frac12.
\]
By homogeneity,
\[
\kappa_\alpha^2
=
\inf_{v\in\mathcal C_\alpha(S^\star)\setminus\{0\}}
\frac{\|Xv\|_2^2}{\|v\|_2^2}
=
\inf_{v\in T}\|Xv\|_2^2
\ge
\frac14.
\]
Therefore
\[
\kappa_\alpha^2\ge\frac14.
\]

Finally, by Cauchy--Schwarz,
\[
\|v_{S^\star}\|_{1,\alpha}
\le
\sqrt{
\frac{a_0}{\alpha^2}+m_{\rm miss}
}\,
\|v_{S^\star}\|_2
\le
\sqrt{s_\alpha}\,\|v\|_2.
\]
Consequently,
\[
\phi_\alpha^2
=
\inf_{v\in\mathcal C_\alpha(S^\star)\setminus\{0\}}
\frac{s_\alpha\|Xv\|_2^2}
{\|v_{S^\star}\|_{1,\alpha}^2}
\ge
\kappa_\alpha^2
\ge
\frac14.
\]
In particular, $\phi_\alpha^2\ge1/4$, as claimed.
\end{proof}

We next control the stochastic term in the basic Lasso inequality.
The dual weighted norm is
$\|z\|_{\infty,\alpha}
=\max\{\alpha\max_{j\in S_{\rm init}}|z_j|,
\max_{j\in S_{\rm init}^c}|z_j|\}$.

\begin{lemma}[Weighted Gaussian score]
\label{lem:weighted-score}
For every $\eta\in(0,1)$, if
$n\ge2\log(2/\eta)$, then with probability at least $1-2\eta$,
\[
\|z\|_{\infty,\alpha}
\le
C\frac{\sigma}{\sqrt n}L_{\alpha,\eta}.
\]
\end{lemma}

\begin{proof}
Since $\|\varepsilon\|_2^2\sim\chi_n^2$,
\citep{laurentmassart2000adaptive} gives
\[
\mathbb P\left(
\|\varepsilon\|_2^2
>
n+2\sqrt{nx}+2x
\right)
\le e^{-x}.
\]
Taking $x=\log(1/\eta)$, the assumption
$n\ge2\log(2/\eta)$ implies
$\|\varepsilon\|_2\le C\sqrt n$ with probability at least $1-\eta$.

Conditionally on $\varepsilon$, the coordinates
$z_j=(\sigma/\sqrt n)X_j^\top\varepsilon$ are independent centered
Gaussians with variance
$\sigma^2\|\varepsilon\|_2^2/n^2\lesssim\sigma^2/n$.
A Gaussian tail bound and separate union bounds over
$S_{\rm init}$ and $S_{\rm init}^c$ give, with conditional
probability at least $1-\eta$,
\[
\alpha\max_{j\in S_{\rm init}}|z_j|
\lesssim
\frac{\sigma}{\sqrt n}
\alpha\sqrt{\log\frac{2|S_{\rm init}|}{\eta}},
\]
and
\[
\max_{j\in S_{\rm init}^c}|z_j|
\lesssim
\frac{\sigma}{\sqrt n}
\sqrt{\log\frac{2|S_{\rm init}^c|}{\eta}}.
\]
Combining these bounds with the Laurent--Massart event proves the
claim.
\end{proof}

\begin{lemma}[Weighted-Lasso estimation]
\label{lem:weighted-estimation}
On the event $\lambda\ge2\|z\|_{\infty,\alpha}$,
\[
\|\widehat\beta^{\rm WL}-\beta^\star\|_2
\le
\frac{3\sqrt{s_\alpha}\lambda}
     {\kappa_\alpha\phi_\alpha}.
\]
\end{lemma}

\begin{proof}
Set $\Delta:=\widehat\beta^{\rm WL}-\beta^\star$.
Optimality of $\widehat\beta^{\rm WL}$ gives
\[
\frac12\|X\Delta\|_2^2
\le
z^\top\Delta
+
\lambda\left(
\|\beta^\star\|_{1,\alpha}
-
\|\beta^\star+\Delta\|_{1,\alpha}
\right).
\]
By weighted Hölder,
$z^\top\Delta\le(\lambda/2)\|\Delta\|_{1,\alpha}$.
Since $\beta^\star$ is supported on $S^\star$, decomposability gives
\[
\|\beta^\star\|_{1,\alpha}
-
\|\beta^\star+\Delta\|_{1,\alpha}
\le
\|\Delta_{S^\star}\|_{1,\alpha}
-
\|\Delta_{(S^\star)^c}\|_{1,\alpha}.
\]
Therefore
\[
\frac12\|X\Delta\|_2^2
\le
\frac{3\lambda}{2}
\|\Delta_{S^\star}\|_{1,\alpha}
-
\frac{\lambda}{2}
\|\Delta_{(S^\star)^c}\|_{1,\alpha}.
\]
Since the left-hand side is nonnegative,
$\|\Delta_{(S^\star)^c}\|_{1,\alpha}
\le3\|\Delta_{S^\star}\|_{1,\alpha}$, hence
$\Delta\in\mathcal C_\alpha(S^\star)$.

By definition of $\phi_\alpha$,
\[
\|\Delta_{S^\star}\|_{1,\alpha}
\le
\frac{\sqrt{s_\alpha}}{\phi_\alpha}
\|X\Delta\|_2.
\]
Substitution in the previous inequality gives
\[
\|X\Delta\|_2
\le
\frac{3\sqrt{s_\alpha}\lambda}{\phi_\alpha}.
\]
Since $\Delta$ belongs to the weighted cone,
$\kappa_\alpha\|\Delta\|_2\le\|X\Delta\|_2$, which proves the result.
\end{proof}

Thus, on the event of \Cref{lem:weighted-geometry} $\phi_\alpha$ and $\kappa_\alpha$ are bounded below
by numerical constants and
$\lambda\ge C\sigma L_{\alpha,\eta}/\sqrt n$,
\[
\|\widehat\beta^{\rm WL}-\beta^\star\|_2
\lesssim
\frac{\sigma}{\sqrt n}
\sqrt{s_\alpha}\,L_{\alpha,\eta}.
\]

\subsection{Exact signed support recovery}

Set $H:=X^\top X$ and
$\rho:=\operatorname{sign}(\beta^\star_{S^\star})$.
The proof uses a primal-dual witness: we impose support $S^\star$ and
signs $\rho$, solve the KKT equations on $S^\star$, and verify strict
dual feasibility on $(S^\star)^c$.

Let $W$ be the diagonal matrix of weights $[W]_{ii} = \frac{1}{\alpha} $ if $i \in S_{init}$, else $1$.

For the recovery bounds define
\[
L_{\alpha,\eta}^{\rm rec}
:=
\max\left\{
\alpha
\sqrt{
\log\frac{2|S_{\rm init}\setminus S^\star|}{\eta}
},
\sqrt{
\log\frac{
2|S_{\rm init}^c\cap(S^\star)^c|
}{\eta}
}
\right\},
\]
with the corresponding term omitted when one of the two sets is empty.

We first control the Gram matrix on the true support.

\begin{lemma}[Gram matrix on $S^\star$]
\label{lem:wl-gram}
There exists a universal constant $C>0$ such that, if
\[
n\ge
C\left(
s+\log\frac2\eta
\right),
\]
then, with probability at least $1-\eta$,
\[
\frac12 I
\preceq
H_{S^\star,S^\star}
\preceq
\frac32 I.
\]
In particular,
$\|H_{S^\star,S^\star}^{-1}\|_{\rm op}\le2$.
\end{lemma}

\begin{proof}
The matrix $\sqrt n\,X_{S^\star}$ has i.i.d. standard Gaussian
entries. The usual upper and lower singular-value bounds give
\[
1-C\sqrt{\frac{s}{n}}-C\sqrt{\frac{\log(2/\eta)}{n}}
\le
\lambda_{\min}^{1/2}(H_{S^\star,S^\star})
\]
and the corresponding upper bound for
$\lambda_{\max}^{1/2}(H_{S^\star,S^\star})$, with probability at
least $1-\eta$. Increasing the universal constant in the sample-size
condition gives the claim.
\end{proof}

The next bound controls the deterministic part of the inactive KKT
conditions.

\begin{lemma}[Weighted irrepresentability]
\label{lem:wl-irrepresentability}
On the event of \Cref{lem:wl-gram}, conditionally on $X_{S^\star}$,
with probability at least $1-\eta$,
\[
\left\|
W_{(S^\star)^c}^{-1}
H_{(S^\star)^c,S^\star}
H_{S^\star,S^\star}^{-1}
W_{S^\star}\rho
\right\|_\infty
\le
C\sqrt{\frac{s_\alpha}{n}}\,
L_{\alpha,\eta}^{\rm rec}.
\]
In particular, the left-hand side is at most $1/2$ provided
\[
n
\ge
Cs_\alpha
\left(L_{\alpha,\eta}^{\rm rec}\right)^2.
\]
\end{lemma}

\begin{proof}
Let
$u:=X_{S^\star}
H_{S^\star,S^\star}^{-1}W_{S^\star}\rho$.
On the Gram event,
\[
\|u\|_2^2
=
\rho^\top
W_{S^\star}
H_{S^\star,S^\star}^{-1}
W_{S^\star}\rho
\le
2s_\alpha.
\]
Conditionally on $X_{S^\star}$, every inactive column $X_j$ is
independent of $u$, and
\[
X_j^\top u
\sim
\mathcal N\left(
0,\frac{\|u\|_2^2}{n}
\right).
\]

For $j\in S_{\rm init}\setminus S^\star$, division by the weight
$w_j=\alpha^{-1}$ produces the factor $\alpha$, so a Gaussian union
bound gives
\[
\max_{j\in S_{\rm init}\setminus S^\star}
\frac{|X_j^\top u|}{w_j}
\lesssim
\sqrt{\frac{s_\alpha}{n}}\,
\alpha
\sqrt{
\log\frac{2|S_{\rm init}\setminus S^\star|}{\eta}
}.
\]
For $j\in S_{\rm init}^c\cap(S^\star)^c$, one has $w_j=1$, hence
\[
\max_{j\in S_{\rm init}^c\cap(S^\star)^c}
|X_j^\top u|
\lesssim
\sqrt{\frac{s_\alpha}{n}}
\sqrt{
\log\frac{
2|S_{\rm init}^c\cap(S^\star)^c|
}{\eta}
}.
\]
The two bounds hold simultaneously with probability at least
$1-\eta$, after adjusting numerical constants, and give the claim.
\end{proof}

We now control the stochastic part of the inactive KKT conditions.

\begin{lemma}[Inactive noise]
\label{lem:wl-inactive-noise}
If $n\ge2\log(2/\eta)$, then, with probability at least $1-2\eta$,
\[
\left\|
W_{(S^\star)^c}^{-1}
\left(
z_{(S^\star)^c}
-
H_{(S^\star)^c,S^\star}
H_{S^\star,S^\star}^{-1}
z_{S^\star}
\right)
\right\|_\infty
\le
C\frac{\sigma}{\sqrt n}
L_{\alpha,\eta}^{\rm rec}.
\]
\end{lemma}

\begin{proof}
Let
$P_{S^\star}^\perp:=
I-X_{S^\star}
H_{S^\star,S^\star}^{-1}
X_{S^\star}^\top$.
Then
\[
z_{(S^\star)^c}
-
H_{(S^\star)^c,S^\star}
H_{S^\star,S^\star}^{-1}z_{S^\star}
=
\frac{\sigma}{\sqrt n}
X_{(S^\star)^c}^\top
P_{S^\star}^\perp\varepsilon.
\]
Conditionally on $(X_{S^\star},\varepsilon)$, every inactive column is
independent of $P_{S^\star}^\perp\varepsilon$. Hence each coordinate
is centered Gaussian with variance at most
$\sigma^2\|\varepsilon\|_2^2/n^2$.

As in \Cref{lem:weighted-score}, \citep{laurentmassart2000adaptive} gives
$\|\varepsilon\|_2\lesssim\sqrt n$ with probability at least
$1-\eta$. On this event, separate Gaussian union bounds over
$S_{\rm init}\setminus S^\star$ and
$S_{\rm init}^c\cap(S^\star)^c$ yield respectively
\[
C\frac{\sigma}{\sqrt n}
\alpha
\sqrt{
\log\frac{2|S_{\rm init}\setminus S^\star|}{\eta}
}
\]
and
\[
C\frac{\sigma}{\sqrt n}
\sqrt{
\log\frac{
2|S_{\rm init}^c\cap(S^\star)^c|
}{\eta}
}.
\]
Taking their maximum proves the result.
\end{proof}

The last ingredient is a beta-min condition ensuring that the active
coefficients keep the correct signs. 
\begin{assumption}[Beta-min condition for weighted-Lasso recovery]
\label{ass:wl-betamin}
There exists a universal constant $C>0$ such that
\[
\beta_{\min}
:=
\min_{j\in S^\star} |\beta_j^\star|
\ge
C\frac{\sigma}{\sqrt n}
\sqrt{\log\frac{2s}{\eta}}
+
C\lambda\sqrt{s_\alpha}.
\]
\end{assumption}

\begin{lemma}[Active-coordinate control]
\label{lem:wl-active}
On the event of \Cref{lem:wl-gram}, with probability at least
$1-\eta$,
\[
\left\|
H_{S^\star,S^\star}^{-1}z_{S^\star}
\right\|_\infty
\le
C\frac{\sigma}{\sqrt n}
\sqrt{\log\frac{2s}{\eta}}.
\]
Moreover,
\[
\left\|
H_{S^\star,S^\star}^{-1}
W_{S^\star}\rho
\right\|_\infty
\le
2\sqrt{s_\alpha}.
\]
Consequently,
\[
\left\|
H_{S^\star,S^\star}^{-1}
\left(
z_{S^\star}
-\lambda W_{S^\star}\rho
\right)
\right\|_\infty
\le
C\frac{\sigma}{\sqrt n}
\sqrt{\log\frac{2s}{\eta}}
+
2\lambda\sqrt{s_\alpha}.
\]
\end{lemma}

\begin{proof}
Conditionally on $X_{S^\star}$,
\[
H_{S^\star,S^\star}^{-1}z_{S^\star}
\sim
\mathcal N\left(
0,
\frac{\sigma^2}{n}
H_{S^\star,S^\star}^{-1}
\right).
\]
On the Gram event, every diagonal element of
$H_{S^\star,S^\star}^{-1}$ is at most $2$, so a Gaussian union bound
over $S^\star$ gives the first inequality.

For the second,
\[
\left\|
H_{S^\star,S^\star}^{-1}
W_{S^\star}\rho
\right\|_\infty
\le
\|H_{S^\star,S^\star}^{-1}\|_{\rm op}
\|W_{S^\star}\rho\|_2
\le
2\sqrt{s_\alpha}.
\]
The last inequality follows by the triangle inequality.
\end{proof}

We can now verify the KKT conditions of the full weighted-Lasso
problem.

\begin{lemma}[Primal-dual witness]
\label{lem:wl-pdw}
Assume the Gram event of \Cref{lem:wl-gram}, the bound in
\Cref{lem:wl-irrepresentability} is at most $1/2$, and the bound in
\Cref{lem:wl-inactive-noise} is at most $\lambda/4$. Under \cref{ass:wl-betamin}, the weighted Lasso recovers the exact signed support:
\[
\operatorname{supp}(\widehat\beta^{\rm WL})=S^\star,
\qquad
\operatorname{sign}
(\widehat\beta^{\rm WL}_{S^\star})
=
\rho.
\]
\end{lemma}

\begin{proof}
Set $\widetilde\beta_{(S^\star)^c}:=0$ and define
\[
\widetilde\beta_{S^\star}
:=
\beta^\star_{S^\star}
+
H_{S^\star,S^\star}^{-1}
\left(
z_{S^\star}
-\lambda W_{S^\star}\rho
\right).
\]
By \Cref{lem:wl-active} and the beta-min assumption,
$\operatorname{sign}(\widetilde\beta_{S^\star})=\rho$.

The definition of $\widetilde\beta$ gives
\[
X_{S^\star}^\top(y-X\widetilde\beta)
=
\lambda W_{S^\star}\rho,
\]
so the active KKT conditions hold. On the inactive coordinates,
\begin{align*}
&
W_{(S^\star)^c}^{-1}
X_{(S^\star)^c}^\top(y-X\widetilde\beta)
\\
&=
W_{(S^\star)^c}^{-1}
\left(
z_{(S^\star)^c}
-
H_{(S^\star)^c,S^\star}
H_{S^\star,S^\star}^{-1}z_{S^\star}
\right)
\\
&\quad+
\lambda
W_{(S^\star)^c}^{-1}
H_{(S^\star)^c,S^\star}
H_{S^\star,S^\star}^{-1}
W_{S^\star}\rho.
\end{align*}
The inactive-noise and irrepresentability bounds therefore give
\[
\left\|
W_{(S^\star)^c}^{-1}
X_{(S^\star)^c}^\top(y-X\widetilde\beta)
\right\|_\infty
\le
\frac{\lambda}{4}
+
\frac{\lambda}{2}
<
\lambda.
\]
Thus the inactive dual variables are strictly feasible and
$\widetilde\beta$ satisfies the KKT conditions of the full problem.

All weighted-Lasso minimizers have the same fitted value and therefore
the same residual. The strict inactive inequalities force every
minimizer to vanish on $(S^\star)^c$. Since
$H_{S^\star,S^\star}$ is invertible, the restricted problem on
$S^\star$ is strictly convex, so the minimizer is unique. This proves
the claim.
\end{proof}

Combining the preceding lemmas, exact signed recovery holds with
probability at least $1-C\eta$ whenever
\[
n
\ge
C\left[
s+\log\frac1\eta
+
s_\alpha
\left(
L_{\alpha,\eta}^{\rm rec}
\right)^2
\right],
\]
\[
\lambda
\ge
C\frac{\sigma}{\sqrt n}
L_{\alpha,\eta}^{\rm rec},
\qquad
\beta_{\min}
\ge
C\frac{\sigma}{\sqrt n}
\sqrt{\log\frac{2s}{\eta}}
+
C\lambda\sqrt{s_\alpha}.
\]

\subsection{Choice of the weight parameter}

We finally optimize the weighted complexity
$s_\alpha (L_{\alpha,\delta}^{\rm rec})^2$ appearing in the above
bound. Let
\[
\ell_{\rm out}
:=
\begin{cases}
\log\!\left(\dfrac{4|S_{\rm init}^c\cap(S^\star)^c|}{\delta}\right),
& S_{\rm init}^c\cap(S^\star)^c\neq\varnothing,\\[1mm]
0, & \text{otherwise},
\end{cases}
\]
and
\[
\ell_{\rm in}
:=
\begin{cases}
\log\!\left(\dfrac{4|S_{\rm init}\setminus S^\star|}{\delta}\right),
& S_{\rm init}\setminus S^\star\neq\varnothing,\\[1mm]
0, & \text{otherwise}.
\end{cases}
\]
 Then
\begin{equation}\label{eq:maximum}
s_\alpha
\left(L_{\alpha,\delta}^{\rm rec}\right)^2
=
\left(
m_{\rm miss}
+
\frac{a_0}{\alpha^2}
\right)
\max\{
\ell_{\rm out},
\alpha^2\ell_{\rm in}
\}.
\end{equation}

If $\ell_{\rm in}=0$, the second inactive block is absent and the
right-hand side decreases with $\alpha$, so the oracle optimum is attained
in the limit $\alpha\to\infty$. If $\ell_{\rm in}>0$, the same
piecewise minimization as above gives
\[
\alpha_{\rm or}^2
=
\max\left\{
1,\frac{\ell_{\rm out}}{\ell_{\rm in}}
\right\}.
\]
Consequently,
\begin{align}\label{eq:weightlassoideal}
\inf_{\alpha\ge1}
s_\alpha
\left(L_{\alpha,\delta}^{\rm rec}\right)^2
&=
m_{\rm miss}\max\{\ell_{\rm out},\ell_{\rm in}\}
+
a_0\ell_{\rm in}.
\end{align}
In the regime $\ell_{\rm out}\ge\ell_{\rm in}$, which is the regime of
interest in our work, this becomes
\[
\inf_{\alpha\ge1}
s_\alpha
\left(L_{\alpha,\delta}^{\rm rec}\right)^2
=
m_{\rm miss}\ell_{\rm out}
+
a_0\ell_{\rm in}.
\]
If $\ell_{\rm in}=0$, the second inactive block is absent and the right-hand side decreases with $\alpha$. Hence
\[
\inf_{\alpha\geq 1}
s_\alpha\left(L_{\alpha,\delta}^{\rm rec}\right)^2
=
m_{\rm miss}\ell_{\rm out},
\]
with the infimum achieved in the limiting sense $\alpha\to\infty$. In particular, when $S_{\rm init}\subseteq S^\star$, there is no inherited-false-positive block, and the oracle choice corresponds to $\alpha_{\rm or}=+\infty$: initialized coordinates are left unpenalized,
while only the $m_{\rm miss}=|S^\star\setminus S_{\rm init}|$ missing
coordinates contribute to the support-search complexity.

This choice is oracle because $\ell_{\rm out}$ and $\ell_{\rm in}$ depend on the unknown support $S^\star$.

This optimization concerns only the weighted logarithmic contribution. Independently, the design event used in \Cref{lem:weighted-geometry} requires $n\gtrsim s+\log(1/\delta)$. This term is dominated in the regime emphasized in the main text, but must be retained in degenerate edge cases; in particular, when $S_{\rm init}=S^\star$, the optimized weighted contribution vanishes and the remaining requirement is $n\gtrsim s+\log(1/\delta)$.

For a computable choice, use $|S_{\rm init}\setminus S^\star|\le|S_{\rm init}|$ and $|S_{\rm init}^c\cap(S^\star)^c|\le|S_{\rm init}^c|$, so that $L_{\alpha,\delta}^{\rm rec}\le L_{\alpha,\delta}$. Since $|S_{\rm init}|\le d/2$, the same minimization gives
\[
\alpha_\star^2
=
\frac{
\log(4|S_{\rm init}^c|/\delta)
}{
\log(4|S_{\rm init}|/\delta)
}
\ge1.
\]
At this value,
\[
L_{\alpha_\star,\delta}^2
=
\log\frac{4|S_{\rm init}^c|}{\delta},
\]
and
\[
s_{\alpha_\star}
L_{\alpha_\star,\delta}^2
=
m_{\rm miss}
\log\frac{4|S_{\rm init}^c|}{\delta}
+
a_0
\log\frac{4|S_{\rm init}|}{\delta}.
\]
Taking $\eta$ to be a sufficiently small numerical multiple of $\delta$, the failure probabilities above sum to at most $\delta$, while replacing $\eta$ by this multiple only changes universal constants in the logarithms. Hence the computable choice
\[
\alpha_\star^2
=
\frac{
\log(4|S_{\rm init}^c|/\delta)
}{
\log(4|S_{\rm init}|/\delta)
}
\]
and
\[
\lambda
=
C\sigma
\sqrt{
\frac{
\log(4|S_{\rm init}^c|/\delta)
}{n}
}
\]
give that \Cref{ass:wl-betamin} holds under the specified sample complexity. It then yields
\[
\|\widehat\beta^{\rm WL}-\beta^\star\|_2
\lesssim
\frac{\sigma}{\sqrt n}
\sqrt{
m_{\rm miss}
\log\frac{4|S_{\rm init}^c|}{\delta}
+
a_0
\log\frac{4|S_{\rm init}|}{\delta}
},
\]
provided
\[
n
\gtrsim
m_{\rm miss}
\log\frac{4|S_{\rm init}^c|}{\delta}
+
a_0
\log\frac{4|S_{\rm init}|}{\delta}.
\]

Indeed, by \Cref{lem:weighted-geometry}, it is sufficient that
\[
n
\gtrsim
s
+
s_{\alpha_\star}
\max\left\{
\alpha_\star^2
\log\!\bigl(1+|S_{\rm init}\setminus S^\star|\bigr),
\log\!\bigl(1+|(S_{\rm init}\cup S^\star)^c|\bigr)
\right\}
+
\log\frac1\delta.
\]
Using
\[
\max(1,|S_{\rm init}\setminus S^\star|)
\le |S_{\rm init}|,
\qquad
\max(1,|(S_{\rm init}\cup S^\star)^c|)
\le |S_{\rm init}^c|,
\]
and the definition of $\alpha_\star$, the maximum in \Cref{eq:maximum} is bounded,
up to universal constants, by
\[
\log\frac{4|S_{\rm init}^c|}{\delta}.
\]
Therefore
\[
s_{\alpha_\star}
\log\frac{4|S_{\rm init}^c|}{\delta}
=
m_{\rm miss}
\log\frac{4|S_{\rm init}^c|}{\delta}
+
a_0
\log\frac{4|S_{\rm init}|}{\delta}.
\]
Since this quantity also controls $s$ up to universal constants, the preceding condition reduces to the stated sample-size requirement.

We also need the beta-min condition of \Cref{lem:wl-pdw} to have exact recovery
\[
\beta_{\min}
\gtrsim
\frac{\sigma}{\sqrt n}
\left[
\sqrt{\log\frac{4s}{\delta}}
+
\sqrt{
m_{\rm miss}
\log\frac{4|S_{\rm init}^c|}{\delta}
+
a_0
\log\frac{4|S_{\rm init}|}{\delta}
}
\right].
\]
Under these conditions,
$\operatorname{supp}(\widehat\beta^{\rm WL})=S^\star$ and the signs
are recovered exactly.

For $\sigma=0$, fix $\alpha=\alpha_\star$. On the design event of \Cref{lem:weighted-geometry}, \Cref{lem:weighted-estimation} applies for every $\lambda>0$ and gives
\[
\|\widehat\beta^{\mathrm{WL}}_\lambda-\beta^\star\|_2
\le 12\lambda\sqrt{s_{\alpha_\star}}\longrightarrow 0
\qquad\text{as }\lambda\to 0.
\]
Moreover, any minimizer of the weighted $\ell_1$ penalty
subject to $X\beta=y$ differs from $\beta^\star$ by an element of $C_{\alpha_\star}(S^\star)\cap\ker X=\{\mathbf{0}\}$.
Thus this constrained minimizer is unique and coincides
with the limiting estimator, proving the noiseless case.

\section{Experiments}
\label{app:experiments}

All experiments use the Gaussian sparse regression model of \Cref{subsec:regression}. The saddle-to-saddle trajectory is implemented by solving the signed constrained least-squares problem at each visited face and then advancing the dual variable to the next boundary hit. In trajectory plots, we use the saddle index on the horizontal axis rather than accelerated time, since consecutive hitting-time intervals may have very different scales.

\begin{figure}[htbp]
    \centering
    \begin{minipage}[c]{0.49\linewidth}
        \centering
        \includegraphics[width=\linewidth]{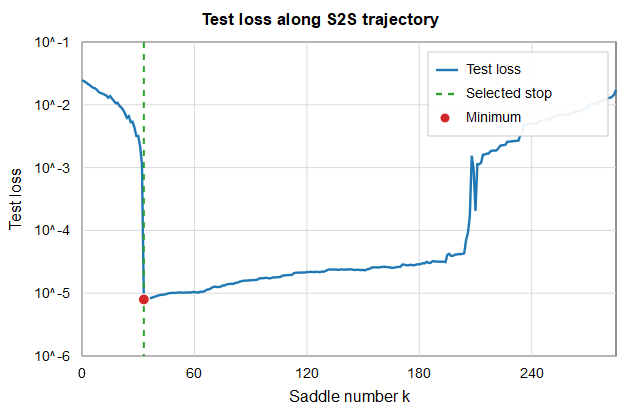}
    \end{minipage}
    \hfill
    \begin{minipage}[c]{0.49\linewidth}
        \centering
        \includegraphics[width=\linewidth]{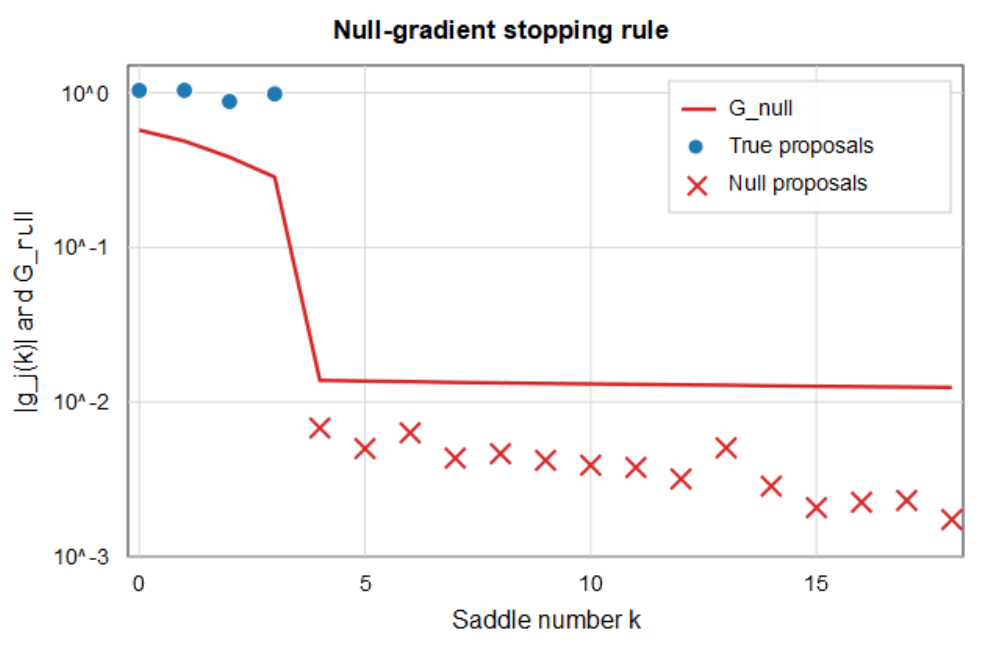}
    \end{minipage}
    \caption{
    Early stopping along S2S trajectories on two independent instances.
    \textbf{Left:} test loss for $d=300$, $n_{\rm train}=220$,
    $n_{\rm test}=1200$, $s=40$, and $\sigma=0.09$. The test loss reaches
    its oracle minimum at saddle $k=33$ and then increases sharply as the
    trajectory continues.
    \textbf{Right:} null-gradient stopping rule for $d=32$, $n=220$,
    $s=6$, and $\sigma=0.05$. The four missing true coordinates are proposed
    before any null coordinate; the first null proposal falls below the
    stopping threshold and is therefore rejected.
    }
    \label{fig:early-stopping}
\end{figure}

\paragraph{Early-stopping illustrations.}
The left panel of \Cref{fig:early-stopping} uses seed $0$, with $g_0=12$, $w_0=0$, initialization scale $0.2$, and $\eta=0.1$. We follow the trajectory for $285$ events, beyond the point that would be useful for prediction. The test loss is minimized at saddle $k=33$, where it equals $7.95\times10^{-6}$. Continuing the trajectory yields a final test loss of $1.73\times10^{-2}$, more than $2\times10^3$ times larger, while the final support contains $297$ of the $300$ coordinates. This panel is not meant to evaluate the stopping rule itself: it illustrates the overfitting that occurs when the S2S trajectory is followed too far, and hence the need for an early-stopping criterion.

The right panel directly illustrates the stopping criterion. We use seed $0$, $d=32$, $n=220$, $s=6$, $g_0=2$, $w_0=f_0=0$, $\sigma=0.05$ and $\eta=0.1$. Thus four true coordinates remain to be learned. At each saddle, we plot the absolute gradient coordinate associated with the next proposed variable together with the uniform null-gradient threshold of \Cref{app:stopping}. The first four proposals are precisely the four missing true coordinates and lie above the threshold. At saddle $k=4$, the first null coordinate is proposed, with $|g_j|=6.79\times10^{-3}$, while $G_{\rm null}=1.38\times10^{-2}$. The stopping rule therefore rejects this proposal and stops before introducing a false positive.

The two panels illustrate the two roles of early stopping along the S2S path. The left panel shows that continuing the trajectory well beyond the useful part of the path can lead to severe overfitting, even though the training procedure keeps activating coordinates. The right panel shows that the null-gradient rule can identify a natural stopping point: the missing true coordinates are accepted first, while the first null proposal falls below the threshold and is rejected. Together, these experiments suggest that the stopping criterion can truncate the trajectory near the statistically relevant part of the path, before spurious activations start to accumulate.

\begin{figure}[htbp]
    \centering
    \ificlr\vspace{-0.8em}\fi
    \includegraphics[width=0.62\linewidth]
    {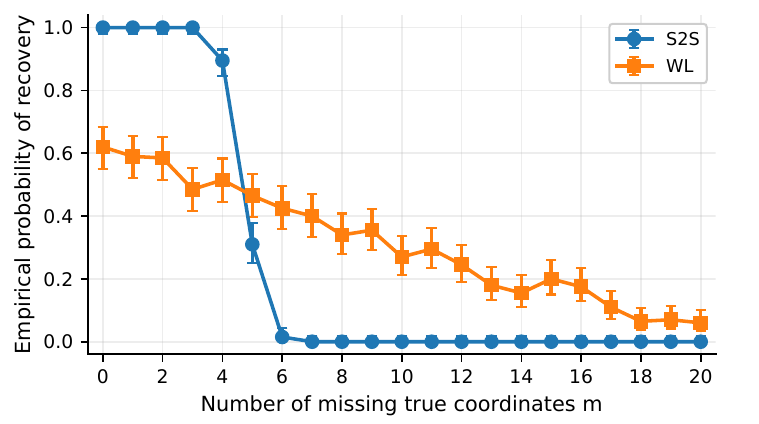}
    \setlength{\abovecaptionskip}{0pt}
    \setlength{\belowcaptionskip}{0pt}
    \caption{
    Empirical probability of exact support recovery as a function of the
    number $m$ of missing true coordinates for a clean initialization
    ($d=260$, $n=140$, $s=20$, $\sigma=0.005$).
    Probabilities are estimated over $200$ independent trials.
    S2S succeeds with high probability when few true coordinates remain to
    be learned, while its recovery probability decreases sharply as $m$
    increases. Weighted Lasso degrades more gradually.
    }
    \label{fig:recovery-m}
    \ificlr\vspace{-0.5em}\fi
\end{figure}

\paragraph{Recovery as a function of initialization quality.}
For \Cref{fig:recovery-m}, we use a clean initialization, so $F_0=\emptyset$, with $d=260$, $n=140$, $s=20$, and $\sigma=0.005$. The nonzero entries of $\beta^\star$ take values in $\{-1,+1\}$. For each value of $m$, the initialization contains $s-m$ true coordinates with their correct signs and no false positive. We use the S2S stopping rule with $\eta=10^{-3}$. The weighted-Lasso baseline uses the balanced weight $\alpha_\star^2= \frac{\log(4|S_{\rm init}^c|/\delta)} {\log(4|S_{\rm init}|/\delta)}$ described in \Cref{sec:early-stopped}. Recovery is declared when the estimated support equals $S^\star$; since the initialization is clean, this is equivalent to $S^\star\subseteq S_{\rm final}\subseteq S^\star\cup F_0$. Each point is averaged over $200$ independent instances and the error bars report Monte Carlo standard errors. Both methods are evaluated on the same generated instances.

This experiment isolates the effect of the number of true coordinates that remain to be learned. As $m$ increases, S2S becomes more likely to stop before all true coordinates have entered: the stopping rule prevents null coordinates from being accepted, but does not by itself guarantee full recovery. Weighted Lasso performs a single global support selection and degrades more gradually in this experiment. Its regularization parameter is set at the theoretical scale rather than tuned by cross-validation, so the comparison concerns the theoretically calibrated procedures rather than optimal predictive performance.

\end{document}